%% file: main.tex
\documentclass{article}

\usepackage[authoryear]{natbib}

\usepackage{amssymb}
\usepackage{amsthm}
\usepackage{hyperref}       
\usepackage{url}            
\usepackage{amsmath,amsfonts,bm}
\usepackage{cleveref}
\usepackage{bbm}
\usepackage{xcolor}         
\usepackage{mathtools}
\usepackage{graphicx}
\usepackage{wrapfig}

\usepackage[old]{old-arrows}
\usepackage{etoc}
\usepackage{setspace}
\usepackage{xspace}
\usepackage[final]{showlabels}

\usepackage{color}

\usepackage{tikz}
\usetikzlibrary{automata,arrows,positioning,calc,shadows.blur}
\usetikzlibrary {arrows.meta}
\input{src/settings/defs}

\usepackage{algorithm}
\usepackage[noend]{algpseudocode}

\usepackage{booktabs}
\usepackage{comment}
\usepackage{array}
\usepackage{nicefrac}

\title{Towards Theoretical Understanding of the Max@K Criterion in Reinforcement Learning}
\title{How to optimize for \maxatk objectives in RL}
\title{The Learnability of \maxatk Reinforcement Learning}
\title{Optimizing \maxatk objectives in Reinforcement Learning}
\title{Learning to Optimize \maxatk in Reinforcement Learning}
\title{Sample-Efficient Reinforcement Learning for \maxatk Objectives}
\title{\maxatk Reinforcement Learning}
\title{On $\mathrm{max}$@$k$ Reinforcement Learning}
\title{Theoretical Foundations of $\max$@$k$ Reinforcement Learning}

\date{}

\author{Riccardo Poiani, Martino Bernasconi \& Andrea Celli\\
Bocconi University, Milan, Italy\\
\texttt{\{riccardo.poiani,martino.bernasconi,andrea.celli2\}@unibocconi.it} 
}

\author{
\begin{tabular}{ccc}
& \\
{Riccardo Poiani}
& 
{Martino Bernasconi} &
{Andrea Celli}\\
\small{Bocconi University} & \small{Bocconi University} & \small{Bocconi University}\\
{\textcolor{black}{\small\texttt{riccardo.poiani@unibocconi.it}}} & %
{\textcolor{black}{\small\texttt{martino.bernasconi@unibocconi.it}}} &
{\textcolor{black}{\small\texttt{andrea.celli2@unibocconi.it}}}
\end{tabular}
}

\begin{document}
\etocdepthtag.toc{mtchapter}
\etocsettagdepth{mtchapter}{subsection}
\etocsettagdepth{mtappendix}{none}

\maketitle

\input{src/abstract}
\input{src/introduction}
\input{src/setting}
\input{src/planning}
\input{src/learning}

\input{src/conclusion}

\bibliographystyle{plainnat}
\bibliography{bibliography}

\appendix
\etocdepthtag.toc{mtappendix}
\etocsettagdepth{mtchapter}{none}
\etocsettagdepth{mtappendix}{subsection}

\appendix
\begin{spacing}{1}
\tableofcontents
\end{spacing}

\input{src/appendix_planning}
\input{src/appendix_learning_old}

\input{src/appendix_aux}

\end{document}

%% file: src/settings/defs.tex
\input{src/settings/colors}
\usepackage{fullpage}
\hypersetup{
	colorlinks = true,
	linkcolor = red!70!black,
	citecolor = blue!75!black,
	linktocpage = true,
	urlcolor = blue!75!black
}

\newtheorem{lemma}{Lemma}

\newtheorem{theorem}{Theorem}

\newtheorem{proposition}{Proposition}

\newcommand{\Ber}{\mathsf{Ber}}

\newcommand{\E}{\mathbb{E}}

\newcommand{\Var}{\mathbb{V}\mathrm{ar}}
\newcommand{\Reals}{\mathbb{R}}
\newcommand{\Naturals}{\mathbb{N}}

\newcommand{\argmin}{\operatornamewithlimits{argmin}}

\newcommand{\Prob}{\mathbb{P}}
\newcommand{\OPT}{\mathsf{Opt}}

\newcommand{\histPol}{\Pi_{\textup{hist}}}
\newcommand{\compPol}{\Pi_{\textup{C}}}
\newcommand{\markPol}{\Pi_{\textup{M}}}

\makeatletter
\newcommand{\mapstoto}{\mathpalette\@mapstoto\relax}
\newcommand*{\@mapstoto}[2]{%
    \mathrel{%
      \vcenter{%
         \vbox{%
            \baselineskip\z@skip
            \lineskip\z@
            \ialign{##\cr$#1\mapstochar\varrightarrow$\cr
            $#1\mapstochar\varrightarrow$\cr}%
         }%
      }%
   }%
}
\makeatother

\newcommand{\ie}{\emph{i.e.},\xspace}
\newcommand{\eg}{\emph{e.g.},\xspace}
\crefname{assumption}{Assumption}{Assumptions}

\usepackage{dsfont}

\protected\def\myfavoriteat{%
  {\fontfamily{ptm}\selectfont @}%
}
\begingroup\lccode`~=`@ \lowercase{\endgroup\let~}\myfavoriteat
\catcode`@=\active

\newcommand{\maxatk}{$\mathrm{max}$@$k$\xspace}
\newcommand{\passatk}{$\mathrm{pass}$@$k$\xspace}

%% file: src/settings/colors.tex
\definecolor{niceRed}{RGB}{190,38,38}
\definecolor{Red2}{RGB}{219, 50, 54}
\definecolor{mgreen}{RGB}{160, 200, 140}
\definecolor{blueGrotto}{RGB}{5,157,192}
\definecolor{limeGreen}{HTML}{81B622}
\definecolor{myellow}{rgb}{0.88,0.61,0.14}
\definecolor{darkGreen}{HTML}{2E8B57}
\definecolor{navyBlueP}{HTML}{03468F}
\definecolor{Sepia}{HTML}{7F462C}
\definecolor{red2}{HTML}{1F462C}
\definecolor{orange2}{HTML}{FF8000}
\definecolor{mgray}{HTML}{ABB3B8}
\definecolor{lgray}{HTML}{E5E8E9}
\definecolor{myPurple}{RGB}{175,0,124}
\definecolor{mypurple2}{rgb}{0.8,0.62,1}
\definecolor{royalBlue}{HTML}{057DCD}
\definecolor{mpink}{HTML}{FC6C85}
\definecolor{lblue}{RGB}{74,144,226}
\definecolor{peagreen}{RGB}{152,193,39}
\definecolor{typ_navy}{HTML}{001f3f}
\definecolor{typ_blue}{HTML}{0074d9}
\definecolor{typ_aqua}{HTML}{7fdbff}
\definecolor{typ_teal}{HTML}{39cccc}
\definecolor{typ_eastern}{HTML}{239dad}
\definecolor{typ_purple}{HTML}{b10dc9}
\definecolor{typ_fuchsia}{HTML}{f012be}
\definecolor{typ_maroon}{HTML}{85144b}
\definecolor{typ_red}{HTML}{ff4136}
\definecolor{typ_orange}{HTML}{ff851b}
\definecolor{typ_yellow}{HTML}{ffdc00}
\definecolor{typ_olive}{HTML}{3d9970}
\definecolor{typ_green}{HTML}{2ecc40}
\definecolor{typ_lime}{HTML}{01ff70}
\definecolor{newgreen}{HTML}{83c702}
\definecolor{newpurp}{RGB}{97,96,121}
\definecolor{brickred}{RGB}{203,65,84}

%% file: src/abstract.tex
\begin{abstract}
Reinforcement Learning is a cornerstone technique for modern large reasoning models. Usually, for difficult tasks such as code generation and theorem proving, the agent is evaluated by generating $K$ responses rather than sampling a single response, and performance is then measured using a retry-aware metric such as $\max$@$k$.
Despite their practical importance, the theoretical foundations of learning under such criteria remain limited. In this work, we provide a theoretical study of the $\max$@$k$ learning problem in finite-horizon reinforcement learning. 
We show that optimizing the $\max$@$k$ objectives is fundamentally different from standard expected-return maximization. In particular, we prove that Markovian policies are in general insufficient, identify a compact state augmentation that restores optimality, and explicitly characterize the performance gap that can arise between history-dependent and non-history-dependent policies. Moreover, we show that learning $\max$@$k$-optimal policies is statistically harder than standard reinforcement learning and provide an efficient algorithm that achieves the optimal sample complexity rate. 
\end{abstract}

%% file: src/introduction.tex
\section{Introduction}

Reinforcement Learning is one of the main drivers of the recent successes of Large Language Models (LLMs) in complex coding and mathematical tasks (see, \eg \citep{jaech2024openai,lambert2024tulu,xin2024deepseek,guo2025deepseek}). When dealing with challenging tasks such as theorem proving and code generation, the LLM performance can improve substantially if the model is allowed to generate several responses and is evaluated on the best one.
Early evidence of this behavior appeared in code generation and mathematics \citep{chen2021evaluating,cobbe2021training,li2022competition,xin2024deepseek}, and subsequent studies showed inference-compute scaling laws across a broad range of tasks and models (see, \eg \citep{wang2022self,brown2024large,snell2025scaling}). 
This metric, which scores the model by the highest reward obtained across $K$ generations, is called \maxatk and extends \passatk beyond binary outcomes \citep{bagirov2025best}.

A growing body of literature has shown that optimizing for the classic single-sample objective and then evaluating with the \maxatk metric has suboptimal performance \citep{cui2025entropy,chen2026does}. 
A natural solution to this mismatch is to optimize the \maxatk objective directly. Recent work has therefore proposed gradient estimators for \maxatk, sometimes referred to as inference-aware estimators, as well as modifications of advantage estimates to better align policy updates with the objective \citep{chow2025inference,chen2025pass,tang2025optimizing,walder2026pass,bagirov2025best,takashiro2026advantage}. While these approaches correctly identify the relevant optimization problem, we still lack a complete theoretical understanding of the \maxatk objective in reinforcement learning. Indeed, most existing RL theory relies, explicitly or implicitly, on the linearity of expected cumulative reward, whereas the \maxatk objective is nonlinear in the distribution over trajectories.

In this paper, we take on the challenging task of analyzing the finite-horizon model under the \maxatk objective from a theoretical perspective, clearly highlighting the differences between classic RL and \maxatk RL. 
Our main contributions can be summarized as follows.
\begin{enumerate}
    \item In \Cref{sec:planning} we show that Markovian policies are not optimal for this problem. 
    This is a pivotal difference between standard RL and RL with \maxatk objectives. We also prove that the gap between these two classes of policies can be large, and that the ratio between them is lower-bounded by a constant greater than $1$. This shows that considering history-dependent policies yields a material improvement in performance over Markovian policies.
    \item In \Cref{sec:policy-space}, we then show that full history dependency is not necessary. Indeed, we demonstrate that for optimal policies, the history can be compressed into only two relevant pieces of information: the maximum reward accrued in past rollouts $\{1,\ldots, j-1\}$, and the cumulative reward inside the current $j$-th rollout. This defines a class of policies that is much more manageable than full history-dependent policies, making working with history-dependent policies not computationally much more demanding than standard RL. 
    \item In \Cref{sec:approx-planning}, we exploit our compressed history representation for the offline planning problem, and we show that we can efficiently approximate optimal history-dependent policies. In particular, we demonstrate that computing an $\epsilon$-optimal policy can be done in time polynomial in $1/\epsilon,K,S,A,H$. This shows that the computational problem of finding an optimal policy remains tractable even if optimal policies cannot be oblivious to history.
    \item In \Cref{sec:hardness}, we show that \emph{approximating} the optimal history-dependent policy is necessary, and that inverse polynomial dependence on the approximation is necessary, since even for inverse exponential approximations the problem becomes NP-hard. This contrasts with what happens in a standard finite-horizon MDP, where the \emph{exact} planning problem can be solved in polynomial time.
    \item Then, in \Cref{sec:lb}, we turn to the learning problem, where we first show that, compared to the classic bound for learning under a generative model in a finite-horizon MDP \citep[\eg][]{gheshlaghi2013minimax,sidford2019nearoptimaltimesamplecomplexities}, the sample complexity suffers an additional multiplicative linear dependence on the number of rollouts $K$. Intuitively, this follows from the fact that small probabilities of reaching a large payoff must be estimated with much greater precision, since even small probabilities can yield a large payoff when sampled multiple times. Formally, we prove a $\Omega(KSAH^3\epsilon^{-2}\log(\delta^{-1}))$ lower bound for any $(\epsilon,\delta)$-correct algorithm. In \Cref{sec:algo} we show a matching $\widetilde O(KSAH^3\epsilon^{-2}\log(\delta^{-1}))$ upper bound for the time homogeneous case and a $\widetilde O(KSAH^4\epsilon^{-2}\log(\delta^{-1}))$ in the inhomogeneous case. In general, this result shows that learning for \maxatk is $K$ times more demanding than classical RL.
\end{enumerate}

We now discuss some high-level connections between our theoretical contributions and some state of the art approaches to reasoning tasks. A recurring recent trend in this literature is inference-time refinement, where multiple candidate solutions are generated sequentially, with later attempts leveraging information from earlier ones. Examples include self-refinement and reflection methods \citep{zelikman2022star,madaan2023self,shinn2023reflexion,dhuliawala2024chain,zhang2025beyond}, search-based reasoning methods \citep{yao2023tree,besta2024graph,xie2024monte}, recursive self-aggregation \citep{venkatraman2025recursive}, and multi-attempt training with verifier feedback \citep{lightman2024let,ildiz2026learning}. Recent proof-oriented systems also use population-level generation, verification, refinement, and selection at test time \citep{chen2026maxproof,tsoukalas2026advancing,chung2026goedel}.
Our theoretical model shows that independent sampling of several responses is quantitatively suboptimal with respect to these recent adaptive approaches. At a conceptual level, these methods can be interpreted as a richer implementation of the idea that history-dependent policy over reasoning traces yields larger performance. At the same time, our compression result shows that the full history is not necessary: for optimal control, it suffices to retain the best return from previous rollouts and the cumulative return from the current rollout. This mirrors, at an abstract level, the practical design principle behind many adaptive reasoning methods: useful information from previous attempts should be retained, but only through a compressed summary rather than by conditioning on the entire history. Overall, our theoretical model not only explores an interesting new objective function for RL motivated by the current ways reasoning models are deployed and evaluated, but also aligns with the evidence and principles developed by practitioners facing these problems.

\paragraph{Additional Related Works} From a more theoretical perspective, our work is related to RL where the objective is a nonlinear function of the policy-induced occupancy measure \citep[\eg][]{hazan2019provably,zhang2020variational,zahavy2021reward,kumar2022policy,mutti2022challenging,mutti2023convex,chen2024robust}. These papers are related as nonlinear objectives might invalidate the optimality of Markovian policies. The main difference is that \maxatk is not a generic function of the state frequencies but a structured order statistic across several rollouts. Indeed, our study and many of our considerations are tailored around our performance criteria. 
Another strand of theoretical literature that is close in spirit to our work is risk-sensitive~\citep{chow2014algorithms,fei2020risk,bastani2022regret,wang2023near} and distributional RL~\citep{bellemare2017distributional,dabney2018distributional,bellemare2023distributional,zhang2023estimation,rowland2024near}. Here, however, the focus is primarly either on the problem of precisely learning the distribution of returns or a generic functional of the return distribution. 
Within this context and close to our work, \citet{wu2023risk} study general utilities of the cumulative reward and recover dynamic programming by augmenting the state with the cumulative reward. Our compressed representation extends this idea to the multi-rollout \maxatk setting, where the policy may adapt both across rollouts and within the current rollout.  
Furthermore, it is worth mentioning the literature that directly studies generic non-Markovian reward decision processes, where the transition dynamics are Markovian but rewards depend on the history~\citep{bacchus1996rewarding,brafman2018ltlf,icarte2022reward,gaon2020reinforcement,trapasso2025model}. Our setting differs in that the one-step reward is Markovian and history dependence directly arises from the \maxatk evaluation criterion. 
Finally, we conclude by discussing related work in bandits. Here, it is relevant to mention  that \citet{tong2026finite} study retry-aware stochastic bandits where the actions are selected according to the posterior expected maximum over virtual retries. Their setting shares with ours the fact that retry-aware objectives value diversity and high-tail outcomes. However, their regret guarantees are stated in terms of the standard regret notion and not w.r.t.~a retry-aware objective. Attention to high-tail outcomes has also been explored in extreme bandits \citep[\eg][]{carpentier2014extreme,nishihara2016no,achab2017max,baudry2022efficient}, where the mean-reward objective is replaced by order-statistic criteria. 

%% file: src/setting.tex
\section{Preliminaries} \label{sec:prelimnaries}

\paragraph{Markov Decision Process.}
A finite horizon \emph{time inhomogeneous} Markov Decision Process \citep[MDP,][]{puterman1990markov} is a tuple $\mathcal{M} =
(\mathcal{S},\mathcal{A},H,\{p_h\}_{h=1}^{H},\{r_h\}_{h=1}^{H})$, where $\mathcal{S}$ is a finite set of states, $\mathcal{A}$ is a finite set of actions, $H \in \Naturals$ is the interaction horizon, $\{p_h\}_{h=1}^H$ is the transition kernel with $p_h: \mathcal{S} \times \mathcal{A} \to \Delta(\mathcal{S})$\footnote{Given a set $\mathcal{X}$, we denote by $\Delta(\mathcal{X})$ the set of probability distributions over $\mathcal{X}$.} assigning to each state-action pair $(s,a)$ a probability distribution $p_h(\cdot\mid s,a)$ over the next states, and $\{r_h\}_{h=1}^H$ is the the reward function with $r_h: \mathcal{S} \times \mathcal{A} \to [0,1]$ assigning the immediate reward obtained after taking action $a$ in state $s$ at step $h$. The MDP $\mathcal{M}$, instead, is said to be \emph{time homogeneous} if the transition kernel and reward function are fixed across the interaction horizon, \ie all $p_h$'s and $r_h$'s are equal to some fixed $p$ and $r$.
A history-dependent stochastic policy $\pi$ models the behavior of the agent and is defined as a sequence $\pi = \{\pi_h\}_{h \in [H]}$, where each $\pi_h$ maps the history $\mathcal{H}_h$ observed up to the beginning of step $h$ to $\Delta(\mathcal{A})$. Precisely, the history $\mathcal{H}_h$ is an element that belongs to $\mathcal{S} \times (\mathcal{S} \times \mathcal{A})^{h-1}$, and the interaction with the environment proceeds as follows. The agent starts in some state $s_1 \in \mathcal{S}$, and for each $h \in [H]$, it selects an action $a_h \sim \pi_h(\cdot \mid \mathcal{H}_{h})$, receives a reward $r_h(s_h, a_h)$ and transitions to the next state $s_{h+1} \sim p_h(\cdot \mid s_h, a_h)$. The interaction with the environment ends after $H$ actions.
Given an initial state $s \in \mathcal{S}$, the objective of the agent is to maximize the expected cumulative reward over the horizon $\mathbb{E}_{\pi}\left[ \sum_{h \in [H]} r_h(s_h, a_h) \mid s_1 = s \right]$, where the expectation is taken with respect to the stochasticity of the environment and the policy. A particularly important class of policies is the Markovian one, for which the action distribution at step $h$ depends only on the current state, \ie $\pi_h(\cdot \mid s_h) $. It is well known that in finite-horizon MDPs, Markovian policies are sufficient for optimality: there always exists a Markovian policy that maximizes the expected cumulative reward.

\paragraph{The \maxatk Decision Process.} 
In this work, we extend the above formulation by considering an agent with a budget of $K \in \Naturals$ rollouts. For a starting state $s \in \mathcal{S}$, the agent interacts with the MDP for $K \in \Naturals$ rollouts, each of length $H$. At the beginning of every rollout, the environment is reset to the initial state $s$, while the agent retains the complete history collected during all previous rollouts. Formally, we denote by $\mathcal{H}_{h,i} \in (\mathcal{S} \times (\mathcal{S} \times \mathcal{A})^H)^{i-1} \times (\mathcal{S} \times (\mathcal{S} \times \mathcal{A})^{h-1})$ the \emph{entire} history observed up to the beginning of step $h$ in rollout $i$. Then, the behavior of the agent is modeled by  history-dependent stochastic policy $\pi = \{\pi_{h,i}\}_{h \in [H], i \in [K]}$, where each $\pi_{h,i}$ maps the history $\mathcal{H}_{h,i}$
to a distribution over actions. In this document, we denote by $\histPol$ the set of all the history dependent policies. For an initial state $s \in \mathcal{S}$, the performance of any policy $\pi \in \histPol$ is no longer the expected return of a single trajectory, but the \emph{expected best return} observed among the $K$ rollouts, \ie
\begin{align}\label{eq:objective}
    \E_{\pi} \left[ \max_{i \in [K]} \sum_{h=1}^H r_h(s_{h,i}, a_{h,i})  \mid s_{1,i} = s~\forall i \in [K] \right],
\end{align}
where the expectation is taken with respect to the joint distribution over the $K$ rollouts induced by $\pi$ and the stochasticity of the environment, and a state-action pair $(s_{h,i}, a_{h,i})$ denotes the state-action pair of the $i$-th rollout at step $h$.
We first observe that any history dependent policy $\pi$ satisfies the following Bellman recursive equations. Precisely, for a fixed history-dependent policy $\pi$, step $h \in [H+1]$ and rollout index $i \in [K]$, and history $\mathcal{H}_{h,i}$, define 
\[
V^\pi_{h,i}(\mathcal{H}_{h,i}) =  \E_{\pi} \left[ \max_{j\in[K]} \sum_{l=1}^H r_l(s_{l,j}, a_{l,j}) \mid \mathcal{H}_{h,i} \right],
\] 
as the value function conditioned on the history up to the beginning of step $h$ in rollout $i$. Then, an application of the tower law of expectation yields
\[
V^\pi_{h,i}(\mathcal{H}_{h,i}) = \sum_{a \in \mathcal{A}} \pi_{h,i}(a \mid \mathcal{H}_{h,i}) \sum_{s' \in \mathcal{S}} p_h(s' \mid s_{h,i},a) V^\pi_{h+1,i}\left( \mathcal{H}_{h,i} \circ (a,s') \right),
\]
where $\mathcal{H}_{h,i}\circ(a,s')$ denotes the history obtained from $\mathcal{H}_{h,i}$ by appending the state-action pair $(a,s')$ to the history. Now, observe that, at the end of rollout $i < K$, the recursion moves to the next rollout, meaning that $V^\pi_{H+1,i}(\mathcal{H}_{H+1,i}) = V^\pi_{1,i+1}(\mathcal{H}_{1,i+1})$. At the end of the last rollout, we have that $V^\pi_{H+1,K}(\mathcal{H}_{H+1,K}) = \max_{j\in[K]} \sum_{h=1}^H r_h(s_{h,j}, a_{h,j})$. Finally, at the beginning, we have that $V^\pi_{1,1}(s)$ is exactly the expected value of the best return gathered in $K$ rollouts under policy $\pi$ starting from state $s$.

The optimal value function is defined as the supremum over all the history-dependent policies, \ie
$V^\star_{h,i}(\mathcal{H}_{h,i}) = \sup_{\pi \in \histPol} V^\pi_{h,i}(\mathcal{H}_{h,i})$. Via backward induction one can easily show that it satisfies the following Bellman optimality equations for all $h \in [H]$
\[
V^\star_{h,i}(\mathcal{H}_{h,i}) = \max_{a\in\mathcal{A}} \sum_{s'\in\mathcal{S}} p_h(s'\mid s_{h,i},a) V^\star_{h+1,i} \left(\mathcal{H}_{h,i}\circ(a,s')\right),
\]
with the following boundary conditions: $V^\star_{H+1,i}(\mathcal{H}_{H+1,i}) = V^\star_{1,i+1}(\mathcal{H}_{1,i+1})$ for all $i < K$ and, for the last rollout,  $V^\star_{H+1,K}(\mathcal{H}_{H+1,K}) = \max_{j\in[K]} \sum_{h=1}^H r_h(s_{h,j}, a_{h,j})$. An optimal deterministic history-dependent policy is obtained by choosing, for every $\mathcal{H}_{h,i}$ with $h \in [H]$, any action that attains the max in the definition of the optimal value function.

%% file: src/planning.tex
\section{Planning}\label{sec:planning}

In this section, we study the problem of computing an optimal policy for the \maxatk problem when $\mathcal{M}$ is known and given as input to the agent. Precisely, in \Cref{sec:policy-space}, we first show that the \maxatk criteria change the nature of the planning problem as Markovian policies are no longer sufficient for optimality. Furthermore, we prove that the full history is unnecessary for the agent: it suffices to keep track of the current state, the best return obtained so far, and the cumulative reward in the current rollout. Building on this characterization, in \Cref{sec:approx-planning}, we propose an approximate algorithm that yields a polynomial-time approximation scheme for the optimal policy. Finally, in \Cref{sec:hardness}, we show that this is essentially all we can ask for, as we prove that exact planning is computationally hard.

\subsection{On the Policy Space}\label{sec:policy-space}

\paragraph{Markovian Policies Are Not Enough.}
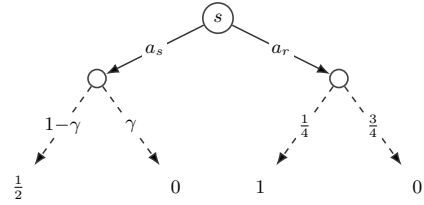
\begin{wrapfigure}{r}{0.4\textwidth}
  \centering
  \scalebox{0.8}{\input{imgs/mdp_small}}
  \caption{
  Visualization of the MDP that illustrates the sub-optimality of Markovian policies. Dashed edges denote stochastic transitions, and leaves are associated with the rewards at the end of the horizon.
  }
  \label{fig:mdpsmall}
\end{wrapfigure}
The main challenge in \maxatk decision processes is that, unlike the classic MDP setting, we are no longer guaranteed to find optimal policies among Markovian policies, denoted in the following as $\markPol$. Within this context, we refer to Markovian policies as policies that depend on the current state $s$, the step $h$, but also the rollout index $i$.\footnote{One can show stronger sub-optimality by restricting attention to policies that depend only on $h$ and $s$.}
We demonstrate the issue with policies in $\markPol$ with a simple example. Consider an MDP, visualized in \Cref{fig:mdpsmall}, in which, at a starting state $s$, there are two actions: one safe, $a_s$, and one risky, $a_r$. The risky action leads to a total reward of $\tfrac12$ with probability $1-\gamma$ for some small $\gamma>0$ (and $0$ otherwise), and the risky action leads to a reward of $1$ with probability $\tfrac14$ (otherwise it gives a reward of zero).\footnote{To implement the reward process that we described, it is sufficient that, when, \eg taking $a_s$, the agent reaches a state where all actions are identical and lead to a leaf with reward $\tfrac12$ and $0$ with probability $1-\gamma$ and $\gamma$, respectively.} Any Markovian policy will either play $a_s$ in one rollout and $a_r$ in the other, or play both $a_s$ and $a_r$.\footnote{Indeed, we can ignore stochastic policies since the objective is multilinear in the policy. Hence, there exists a deterministic policies that attains the optimum within $\markPol$.} It is easy to see that for small $\gamma$ the best is to play once $a_r$ and once $a_s$, which gives an expected reward of $\mathbb{E}[\max(\tfrac12\Ber(1-\gamma),\Ber(\tfrac14))]=\tfrac58-\tfrac{3\gamma}8$.\footnote{With $\Ber(p)$ we indicate a Bernoulli random variable with parameter $p$.}
On the other hand, the non-Markovian policy which plays $a_s$ in the first rollout and $a_r$ in the second rollout only if the first rollout had a reward of $1/2$ will have an expected reward of $\mathbb{E}[\max(\tfrac12,\Ber(\tfrac14))]=\tfrac58$ conditioned on the first rollout being positive (which happens with probability $1-\gamma$), and $\mathbb{E}[\tfrac12\Ber(1-\gamma)]=\tfrac{1-\gamma}2$ conditioned on the first rollout being negative (which happens with probability $\gamma$). In total, this history-dependent policy gives an expected reward of $\tfrac58(1-\gamma)+\gamma\tfrac{1-\gamma}2=\tfrac58-\gamma(\tfrac18+\gamma)>\tfrac58-\tfrac{3\gamma}8$ for small $\gamma>0$. 

\paragraph{The Price of Non-Adaptivity}
We just saw that history-dependent policies $\histPol$ might outperform Markovian ones. We now complement the previous example, with a quantitative analysis of the gap between these two classes of policies. Specifically, we show a multiplicative gap between the performance of history dependent policies and markovian ones is some fixed constant $c>1$. 

\begin{proposition}[Multiplicative Gap of Markovian Policies]\label{prop:comp_LB}
    For any $K\ge 2$, there is an instance $\mathcal{M}$ and a state $s$ for which 
    \[
    \max_{\pi \in \histPol} V^\pi_{1,1}(s) \ge \left(1.24-O(\tfrac1{K})\right)\cdot \max_{\pi \in \markPol} V_{1,1}^\pi(s). 
    \]
\end{proposition}
Recall that a Markovian policy is allowed to depend on the rollout index as well, and thus is a stronger benchmark than using the same Markovian policy repeatedly.
This shows that using Markovian policies generally comes at a high cost, and thus it is of paramount importance learning in a larger policy space.

\paragraph{History Can Be Compressed.} In the following, we show that it is possible to search for an optimal policy in a space which is smaller than $\histPol$. Specifically, we prove that at rollout $i$ and step $h$, it is sufficient to condition $\pi_{h,i}$ on the current state, the maximum return collected in the previous $i-1$ rollouts, and the cumulative reward that has been obtained in the previous $h-1$ steps, at the $i$-th rollout. To formally state the result, let us first introduce some notation. We denote by $G_{h,i}$ the cumulative reward gathered in rollout $i$ at step $h$, namely $G_{h,i} = \sum_{k=1}^h r_h(s_{h,i}, a_{h,i})$.  
Furthermore, we denote by $M_i = \max_{j < i} G_{H,j}$ the maximum return collected in the first $i-1$ rollouts, with the convention that $M_0 = 0$. Then, for each step $h$ and rollout $i$, let $\Phi(\mathcal{H}_{h,i})= \left(s_{h,i}, M_i, G_{h,i} \right)$ be a compressed representation of the history and let $\compPol$ be the set of policies that, at each step, are conditioned only on the compressed history. Then, we prove the following result.

\begin{proposition}[Sufficient conditions for optimality]\label{prop:sufficient-condition-for-optimality}
    In the \maxatk problem, it is without loss of optimality to consider only deterministic policies that belong to $\compPol$.
\end{proposition}

The proof of \Cref{prop:sufficient-condition-for-optimality} is informative for understanding the problem as it works by construction. Specifically, it shows that there exists a policy $\mu^\star \in \compPol$ that attains optimality in terms of \Cref{eq:objective}. To this end, we first prove that value functions of policies in $\compPol$ are only a function of the compressed history $\Phi(\mathcal{H}_{h,i})$. As a consequence the value $V_{h,i}^\mu(\mathcal{H}_{h,i})$ is constant across histories that have the same shared representation. Furthermore, we also observe that, given two histories with the same $\Phi$'s and appending any $(a, s')$ to these histories, we obtain again two histories with the same compressed representation. This, in turn, implies that the optimization problem $\overline{V}^\star_{h,i}(\mathcal{H}_{h,i}) \coloneqq \sup_{\mu \in \compPol} V^\mu_{h,i}(\mathcal{H}_{h,i})$ obeys the following Bellman-like optimality equations
\begin{align}\label{eq:compressed-bellman-opt}
    \overline{V}^\star_{h,i}(\mathcal{H}_{h,i}) = \max_{a \in \mathcal{A}} \sum_{s' \in \mathcal{S}} p_h(s' \mid s,a) \overline{V}^\star_{h+1, i}(\mathcal{H}_{h,i} \circ (a, s')). 
\end{align}
Then, the last step of the proof is to show by induction that $\mu^\star$ attains the value of the optimal policy among all history-dependent policies, \ie that $\overline{V}^{\mu^\star} = V^\star$.

\paragraph{Remark.} We conclude this section with the following remark. In \Cref{prop:sufficient-condition-for-optimality}, we showed that the class of policies $\compPol$ is sufficient for optimality. The class of policies $\compPol$ depends on the maximum rewards collected at previous rollouts and the current cumulative reward of the rollout. We complement this result by showing that any policy that removes either the current cumulative reward of the maximum return collected at previous rollouts is sub-optimal.

\begin{proposition}[Ignoring Previous Best Return or Current Cumulative Return is Sub-optimal]\label{prop:sub-optimality}
    There exists an instance in which a history-dependent policy gets an expected best return strictly larger than (i) any policy that depends on the current state, rollout index, rollout step, and best return obtained in the previous rollouts, and (ii) any policy that depends on the current state, rollout index, rollout step, and cumulative return in the current rollout.
\end{proposition}
The proof of this statement relies on an instance similar to that of \Cref{fig:mdpsmall}, except that each rollot begins by randomly receiving a reward of either $0$ or $1$.
The agent then reaches a common decision state $\bar s$, where it chooses between a safe action, which adds a deterministic small reward, and a risky action, which adds a larger reward with probability $1/4$ and reward $0$ otherwise (see \Cref{fig:mdp}). This construction isolates the two pieces of information that are jointly needed for optimality. At $\bar s$, the correct local exploration-exploitation tradeoff depends on the current cumulative return $G$, because the same last action leads to different final rollout returns depending on the progress already made in the current rollout. It also depends on the previous best return $M$, because the value of taking risk is determined by whether the current rollout only needs to beat a low benchmark or must improve on an already large return. In particular, the optimal policy tolerates high variance when the downside is already protected by $M$: if the previous best return is large enough, a failed risky attempt does not reduce the \maxatk value, while a successful one can still improve it. By contrast, when the current rollout is already on track to produce a competitive return, and the previous best return is low, the optimal policy may avoid unnecessary variance and secure a safe improvement. If the policy observes $M$ but not $G$, then it knows how good the current rollout must be, but it cannot distinguish the low and high branches within the rollout, and it must take the same last action in situations where the optimal policy takes different actions. Conversely, if the policy observes $G$ but not $M$, then it can react to the progress of the current rollout, but it must use the same map from $G$ to actions regardless of whether the previous best return is low or high. Therefore, no policy that drops either $G$ or $M$ can implement all optimal continuation decisions simultaneously.

\subsection{Approximate Planning Algorithm}\label{sec:approx-planning} 

Above, we have shown that, when solving the \maxatk problem, one can restrict the search to policies that belong to $\compPol$. However, we observe that solving the optimization problem in this restricted space, \eg via backward induction, is computationally demanding. Indeed, suppose to apply backward induction to find the optimal policy $\mu^\star \in \compPol$. Then, at the beginning of the induction, we need to consider all the possible compressed history representations $\Phi(\mathcal{H}_{H+1,K})$. The cardinality of this set is $S \cdot |\mathcal{B}| \cdot |\mathcal{C}_{H+1}|$, where $\mathcal{B}$ is the set of possible previous best returns, and $\mathcal{C}_h$ is the set of possible partial returns at step $h$. More generally, when solving the maximum over actions in \Cref{eq:compressed-bellman-opt} at step $h$ in rollout $i$, we need to consider all the compressed history $\Phi(\mathcal{H}_{h,i})$, whose size is $S \cdot |\mathcal{B}| \cdot |\mathcal{C}_h|$. From this argument, it follows that the complexity of tabular backward induction is of order $\mathcal{O}\left(K A S^2 |\mathcal{B}| \sum_h |\mathcal{C}_h| \right)$. The main issue, now, is that $|\mathcal{C}_h|$ and $|\mathcal{B}|$ can be exponential in $H$.\footnote{To see this, consider an MDP with two actions, and suppose that, at step $h$, one gives reward $0$ and the other $2^{-h}$. Then, the cardinality of $|\mathcal{C}_h|$ and $\mathcal{B}$ are clearly $2^h$ and $2^H$ respectively.} In the following, we discuss how we solve this problem.

The main idea to circumvent this issue is that, intuitively, small changes both in the optimal return observed so far and in the cumulative reward of the current rollout shuold not lead to significant changes in terms of optimal value that the agent can achieve. 
Therefore, we propose to discretize the compressed history representations $\Phi(\mathcal{H}_{h,i})$ along these two dimensions. To this end, we proceed in the following way. We first discretize the interval $[0,1]$ of possible rewards using a uniform grid and then project the true reward function onto this set. Precisely, for a given $\kappa > 0$, denote by $\mathcal{R}_{\kappa} = \{ \kappa \cdot (j - \tfrac{1}{2})\}_{j=1}^{\lceil 1/\kappa \rceil}$ the uniform discretization of the $[0,1]$ interval. Then, for all state-action pair $(s,a)$ and all steps $h$, let $\tilde r_{h, \kappa}(s,a) = \argmin_{r \in \mathcal{R}_{\kappa}} |r - r_h(s,a)| $. Once this is done, we can solve by backward induction the \maxatk problem for the MDP ${\mathcal{M}}_{\kappa} = (\mathcal{S}, \mathcal{A}, H, \{ p_h \}_{h \in [H]}, \{ \tilde r_{h,\kappa} \}_{h \in [H]})$ where we replaced the true reward function with the discretized one. Let us denote by $\tilde \mu^\star_{\kappa}$ the resulting policy for $\mathcal{M}_\kappa$ and observe that $\tilde \mu^\star_{\kappa}$ can be adopted in $\mathcal{M}$ as-is, with the single remark that the compressed history representation uses the discretized rewards rather than the true ones. The following result summarizes the theoretical guarantees of the proposed method. 

\begin{theorem}[Efficient Approximation Algorithm]\label{theo:approximation}
    Let $\kappa > 0$ and denote by $\tilde \mu^\star_{\kappa}$ the deterministic optimal policy for ${\mathcal{M}}_{\kappa}$ which is found by backward induction on $\mathcal{M}_\kappa$. Then, for any rollout $i \in [K]$, step $h \in [H+1]$ and history $\mathcal{H}_{h,i}$, it holds that:
    \begin{align}\label{eq:approximation-eq-main}
        |V^\star_{h,i}(\mathcal{H}_{h,i}) - V^{\tilde \mu^\star_\kappa}_{h,i}(\mathcal{H}_{h,i}) | \le  2\kappa H.
    \end{align}
    Furthermore, $\tilde \mu^\star_{\kappa}$ can be computed in $\mathcal{O}\left( \frac{K S^2A H^3}{\kappa^2} \right)$ time.
\end{theorem}

\Cref{theo:approximation} shows that the proposed method is an \emph{efficient} approximation algorithm for the \maxatk objective in the true MDP $\mathcal{M}$. Indeed, setting $\kappa = \epsilon/H$, we have that we can learn an $\epsilon$-optimal policy for $\mathcal{M}$ in $\mathcal{O}(H^5 S^2 A \epsilon^{-2})$ time, which is polynomial in the size of the instance. The underlying result behind \Cref{eq:approximation-eq-main} was indeed to show that, for fixed a policy $\pi$, evaluated in $\mathcal{M}$ and $\mathcal{M}_\kappa$, yields a performance difference that is at most $H\kappa$ (\Cref{lemma:same-pol-diff-rew}). Then, a simple error decomposition argument yields \Cref{eq:approximation-eq-main}. Concerning the efficiency of the algorithm, instead, it follows the fact that now both $|\mathcal{B}|$ and $| \mathcal{C}_h|$ are bounded by $\mathcal{O}(\tfrac{H}{\kappa})$. Hence, when using backward induction, only a polynomial number of states needs to be considered.

\subsection{Impossibility of an Exact Planning Algorithm}\label{sec:hardness}

In \Cref{sec:approx-planning}, we showed how to obtain an $\epsilon$-approximation of the planning problem in time inversely polynomial in the approximation error. In this section, we show that this is necessary, and we cannot hope to solve the exact planning problem efficiently. We use tools from computational complexity to establish this lower bound. We reduce from the problem of \textsc{Subset-Sum}, which is the problem of deciding if there exists a subset of natural numbers from a given multiset $S\subset \mathbb{N}$ that sums up to a given threshold $L\in \mathbb{N}$. We consider a decision problem of selecting an optimal action given a certain history.
\begin{theorem}[Hardness of Exact Planning]\label{th:NP-hard}
    There is a polynomial time reduction from \textsc{Subset-Sum} to the following problem: given $\eta>0$, an MDP $\mathcal{M}$, a history $\mathcal{H}_{h,i}$ and an action $a\in \mathcal{A}$ decide if $V_{h,i}^{\pi}(\mathcal{H}_{h,i})\ge V_{h,i}^{\pi'}(\mathcal{H}_{h,i})-\eta$ 
    or 
    $V_{h,i}^{\pi}(\mathcal{H}_{h,i})< V_{h,i}^{\pi'}(\mathcal{H}_{h,i})-\eta$, where $\pi$ is a policy playing deterministically action $a$ conditioned on the history $\mathcal{H}_{h,i}$ and $\pi'$ is any policy playing deterministically any action different than $a$. 
\end{theorem}

The practical takeaway from this result is that we cannot hope to solve the planning problem efficiently exactly, and that the fully polynomial approximation scheme of \Cref{sec:approx-planning} is thus tight in this sense.
Moreover, the gap $\eta$ used in the proof of \Cref{th:NP-hard} is exponentially small (in the description of the problem). Such a small gap is actually needed as any hardness for a polynomial gap would contradict the existence of the approximation algorithm of \Cref{sec:approx-planning}.

Interestingly, the proof of \Cref{th:NP-hard} already holds for $K=2$ and when the set of policies $\histPol$ has constant size. This means that what is truly difficult is not computing an optimal policy but even knowing the value of a given one, showing that the exact problem is not even in \textsf{NP}. Given the existence of a fully polynomial approximation scheme, our primary interest was to show that no better algorithm could be found, rather than to provide a full characterization of the problem's computational complexity. However, we would like to point out that the Bellman equations reported in \Cref{sec:prelimnaries} can likely be turned into a \textsf{PSPACE} membership result. We leave to future work the question of whether this problem is actually complete for \textsf{PSPACE}.

%% file: imgs/mdp_small.tex
\begingroup
\thinmuskip=0mu
\medmuskip=0mu
\thickmuskip=0mu

\tikzset{every picture/.style={line width=0.75pt}}

\begin{tikzpicture}[
    node distance=3.1cm,
    >=Latex,
    state/.style={circle,draw=black!75,fill=white,minimum size=2mm,inner sep=3pt,align=center,font=\small},
    tstate/.style={draw=black!00,fill=white,minimum size=7mm,inner sep=0pt,align=center,font=\small},
    redstate/.style={state,draw=red!75!black,very thick,text=red!70!black,minimum size=7mm},
    sinkstate/.style={circle,draw=black!75,fill=white,minimum size=7mm,inner sep=0pt,align=center,font=\small},
    forced/.style={->,double,draw=black!80,line width=0.7pt,double distance=1.4pt},
    blueaction/.style={blue!75!black,->,draw=blue!75!black,very thick},
    greenaction/.style={green!55!black,->,draw=green!55!black,very thick,dashed},
    terminal/.style={->,draw=black!80,line width=0.85pt},
    lab/.style={font=\small,align=center,fill=white,inner sep=1.2pt},
    smalllab/.style={font=\scriptsize,align=center,fill=white,inner sep=1.1pt}
]

\def\downone{-1.}
\def\downtwo{-1.8}
\def\dxone{2}
\def\dxtwo{1.3}
\node[state] (s0) at (0,0) {$s$};

\node[state] (u0) at ($(s0)+(-\dxone,\downone)$) {};
\node[state] (u1) at ($(s0)+(\dxone,\downone)$) {};

\node[tstate] (x1) at ($(u0)+(-\dxtwo,\downtwo)$) {$\tfrac12$};
\node[tstate] (x2) at ($(u0)+(\dxtwo,\downtwo)$) {$0$};
\node[tstate] (x3) at ($(u1)+(-\dxtwo,\downtwo)$) {$1$};
\node[tstate] (x4) at ($(u1)+(\dxtwo,\downtwo)$) {$0$};

\draw[forced] (s0) edge node[lab] {$a_s$} (u0);
\draw[forced] (s0) edge node[lab] {$a_r$} (u1);
\draw[forced,dashed] (u0) edge node[lab] {$1-\gamma$} (x1);
\draw[forced,dashed] (u0) edge node[lab] {$\gamma$} (x2);
\draw[forced,dashed] (u1) edge node[lab] {$\tfrac14$} (x3);
\draw[forced,dashed] (u1) edge node[lab] {$\tfrac34$} (x4);

\end{tikzpicture}
\endgroup

%% file: src/learning.tex
\section{Learning}\label{sec:learning}

In this section, we study the statistical learning problem associated with the \maxatk objective with a generative model. Precisely, in \Cref{sec:pac-framework} we first introduce the PAC framework that we adopt. Then, in \Cref{sec:lb}, we derive a lower bound for the problem, and finally, in \Cref{sec:algo}, we propose a uniform sampling algorithm and analyze its theoretical guarantees.

\subsection{PAC Framework}\label{sec:pac-framework}
We consider the case where the learner has access to a generative model of the environment and seeks to learn an approximately optimal policy for the \maxatk problem in an MDP $\mathcal{M}$ with an unknown transition kernel. The interaction between the agent and the environment is as follows. For a time inhomogeneous MDP, during each step $t \in \Naturals$, the agent selects a state-action pair $(s,a) \in \mathcal{S} \times \mathcal{A}$ and a step $h \in [H]$ and observes a sample $s' \sim p_h(\cdot |s,a)$; for the time homogeneous case, instead, the algorithm simply select $(s,a)$ and observes a sample $s' \sim p(\cdot \mid s,a)$. The learning algorithm takes as input a maximum risk parameter $\delta \in (0,1)$ and an accuracy level $\epsilon$, and its goal is to output a policy $\hat \pi \in \compPol$ such that 
\begin{align}\label{eq:pac-obj}
    \Prob\left( V^\star_{1,1}(s) - V^{\hat \pi}_{1,1}(s) \le \epsilon,~\forall s \in \mathcal{S}  \right) \ge 1- \delta,
\end{align}
while minimizing the number of interactions with the environment. In words, we want to efficiently find an $\epsilon$-optimal policy with high probability. In the following, we refer to algorithms that satisfy \Cref{eq:pac-obj} as $(\epsilon,\delta)$-correct. Finally, we conclude this section with some notation. Specifically, we denote by $\tau_\delta$ the stopping time that controls the end of the data acquisition phase, \ie when the number of samples gathered from the generative model is $\tau_\delta$, the algorithm stops and returns $\hat \pi$. 

\paragraph{Remark on the Generative Model.} The choice of studying learning problems in RL with a generative model has a long-standing history in the RL theory community \citep[see, e.g.,][but also references therein]{gheshlaghi2013minimax,sidford2019nearoptimaltimesamplecomplexities,li2020breaking,agarwal2020model,wang2021sample}. 
This framework is primarily motivated by the fact that it decouples the estimation problem from the exploration challenges that arise in MDPs. Furthermore, it also allows for simpler algorithms and theoretical analyses.
In this work, we considered access to a  generative model to isolate the estimation problem and understand the level of accuracy required to achieve near-optimal performance. 
Extending our results to the standard forward interaction setting \citep[adopted, e.g., in][]{dann2015sample,domingues2021episodic,zhang2025settling} is an important direction for future work.
Nonetheless, as we shall see in this section, adopting a generative model already allows us to understand some important features of the \maxatk problem and to draw a clear distinction with traditional RL.

\subsection{Lower Bound}\label{sec:lb}

We start by deriving a lower bound on the statistical complexity of the problem.

\begin{theorem}[Statistical Lower Bound]\label{thm:lb}
    For every $K \ge 1$, there exists a time-homogeneous MDP instance such that, for any $(\epsilon,\delta)$-correct algorithm, it holds that:
    \begin{align*}
        \E[\tau_\delta] \ge \Omega\left( \frac{KH^3SA}{\epsilon^2} \log\left( \frac{1}{\delta} \right) \right).
    \end{align*}
\end{theorem}

\paragraph{Results.} To better understand \Cref{thm:lb}, it is helpful to first recall the standard result for the classical RL setting with a generative model. In this latter case, for finite-horizon time-homogeneous MDPs, the lower bound is well-known to be $\Omega\left( {H^3SA \epsilon^{-2}} \log\left( \frac{1}{\delta} \right)\right)$ \citep{gheshlaghi2013minimax,sidford2019nearoptimaltimesamplecomplexities}. This rate is also tight in the minimax sense, meaning that there exists an $(\epsilon, \delta)$-correct algorithm that stops with a number of samples that is $\widetilde{\mathcal{O}}\left({H^3SA \epsilon^{-2}} \log\left( \frac{1}{\delta} \right) \right)$ \citep{sidford2019nearoptimaltimesamplecomplexities,li2020breaking}. Therefore, our \Cref{thm:lb} establishes a clear-cut, in terms of statistical efficiency, between standard RL and the \maxatk criterion.\footnote{As a minor note, we observe that \Cref{thm:lb} recovers the standard RL result for the $K=1$ case.} Specifically, it shows that the number of samples required grows linearly in the number of rollouts $K$. Furthermore, $K$ actually multiplies the entire rates of standard RL, meaning that the \maxatk criterion is, from a statistical viewpoint, at least $K$ times more demanding than classical RL. In words, the takeaway from this result is that although the budget $K$ allows the agent to potentially achieve higher returns by using multiple rollouts, identifying a near-optimal behavior for this objective comes at a non-negligible cost.

\paragraph{The Linear Dependency in $K$} 
\begin{wrapfigure}{r}{0.4\textwidth}
  \centering
  \scalebox{0.8}{\input{imgs/mdp_hard}}
  \caption{
  Visualization of the MDP for the linear lower bound in $K$. There is a state $s$ and two actions, $a_1$ (blue) and $a_2$ (red), which both lead to rewards of $1$ and $0$ with probabilities $p$ and $1-p$, respectively.
  }
  \label{fig:mdphard}
\end{wrapfigure}
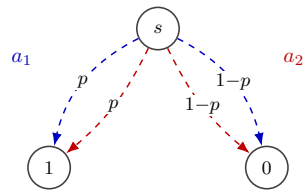
We now provide a sketch behind the argument that allows one to show a linear lower bound in $K$. The example we are going to discuss is illustrative in that it serves as the building block behind the result of \Cref{thm:lb}. Indeed, the proof of \Cref{thm:lb} essentially follows, from a high-level perspective, the same arguments we present now, with the main difference that the presence of the planning horizon $H$ introduces additional technicalities. %
Consider the simple instance $\mathcal{M}_0$  depicted in \Cref{fig:mdphard}. There is a state $s$ with two equal actions $a_1$ and $a_2$. Both actions result in a reward of $1$ with probability $p$ and $0$ with probability $1-p$. Immediately afterward, the rollout terminates and a new one starts from $s$, so the agent repeat this experiment for $K$ times. Next,  consider alternative instances, $\mathcal{M}_1$ and $\mathcal{M}_2$, obtained by making one of the two actions slightly better, \ie by increasing the success probability from $p$ to $p+\Delta$. Here, both $p$ and $\Delta$ should be seen as small numbers, so that the probability of reaching the goal is small in any instance. 
Then, since the learning algorithm needs to return an $\epsilon$-optimal policy, it should be able to distinguish among these instances. Indeed, suppose that $\mathcal{M}_1$ is true, and consider any policy $\pi$. With some algebra and probabilistic arguments, one can show that:
\begin{align}\label{eq:lb-key-eq-main}
V^\star_{1,1}(s) - V^\pi_{1,1}(s) \gtrsim \Delta K  \sum_{i \in [K]} (1 - \alpha_{i}),
\end{align}
where $\alpha_{i}$ is the probability that $\pi$ selects $a_1$ given that the goal was not reached in the $i$-th rollout.\footnote{The intuitive reason behind \Cref{eq:lb-key-eq-main} is that what matters in terms of difference of value functions is just the probability with which the policies reach the goal state (\ie of having one success in $K$ trials). Now, every time that we miss playing $a_1$, we lose a success probability of $\Delta$. Since $p$ and $\Delta$'s are small, the probability that this missed success impacts the value function is bounded away from $0$ and, therefore, these losses accumulate over all rollouts in which the policy does not select $a_1$.} This, in turn, implies that any $(\epsilon, \delta)$-correct algorithm should put enough mass on the optimal action $a_1$ throughout the rollouts, \ie that $m_1 \coloneqq \sum_{i} \alpha_i \gtrsim K - \frac{\Delta}{\epsilon}$. In particular, choosing $\Delta = c \frac{\epsilon}{K}$ for some appropriate constant $c$, one can show that $m_1 \ge \tfrac34 K$, and, thus, for any $(\epsilon, \delta)$-correct algorithm it should hold that $\Prob_{\mathcal{M}_1}(m_1 \ge \tfrac34 K) \ge 1-\delta$. Following an analogous argument, we have that, for $\mathcal{M}_2$, $\Prob_{\mathcal{M}_2}(m_2 \ge \tfrac34 K) \ge 1-\delta$, where $m_2$ is the total mass that the policy puts on $a_2$ across the different rollouts. Clearly, however, only one of these two events holds, as $m_1 + m_2 = K$. Thus, using change-of-measure arguments, the algorithm must collect enough samples to statistically distinguish whether the Bernoulli parameter of $a_1$ or $a_2$ is $p+\Delta$. The information gained from one sample is of order $\textup{KL}(p, p+\Delta) \lesssim \Delta^2/p$, and therefore any correct algorithm requires at least order $p \Delta^{-2} \log\left(1 /\delta \right)$. Choosing $p \approx \tfrac1K$ and since $\Delta \approx \tfrac \epsilon K$, we obtain a lower bound whose order is $K \epsilon^{-2} \log(1/\delta)$.

\paragraph{Technical Remarks.} Before moving to the algorithmic section, we make a remark behind the proof of \Cref{thm:lb}. The main challenge, compared to the lower bound for PAC-RL, arises directly from the \maxatk objective and is already present in the illustrative example that we presented above. The point is that, in RL, taking the wrong action immediately results in a loss of return. In \maxatk, instead, not playing the optimal action has an impact only if the maximum over $K$ rollouts would have changed. Therefore, we need to construct instances in a way that keeps ``successes" sufficiently rare so that not playing the optimal action is decisive across several rollouts. This, in turn, also affects the statistical argument. While in standard RL lower bounds, one often argues that the learning agent should identify the unique optimal action in a certain state, here, the agent can distribute its action choices across several rollouts. As a consequence, we needed to develop an argument that accounts for the total mass assigned to the optimal action across several rollouts. The proof must therefore show that $\epsilon$-optimality forces large inter-rollout mass on the good action. 

\subsection{Algorithm}\label{sec:algo}
We now show that there exists an (efficient) algorithm that attains the same rate of \Cref{thm:lb}. The algorithm combines a standard uniform sampling method \citep[e.g.,][]{gheshlaghi2013minimax} with the approximation scheme we presented in \Cref{sec:planning}. Precisely, the approach works as follows. First, we sample every state-action pair $(s,a) \in \mathcal{S} \times \mathcal{A}$ and every step $h \in [H]$ for $m \in \Naturals$ times (for time-homogeneous systems, instead, we only query every state-action pair $(s,a)$ times). The collected samples are used to construct the empirical transition kernel $\{\hat p _h \}_{h \in [H]}$. Then, the interval $[0,1]$ of possible rewards is discretized uniformly using $\mathcal{O}(\kappa^{-1})$ points and a discretized reward function $\{ \tilde r_{h, \kappa} \}_{h \in [H]}$ is obtained by projecting the true rewards onto the points within the discretization. Once this is done, the algorithm simply stops and return an optimal policy for $\widehat{\mathcal{M}}_\kappa = (\mathcal{S}, \mathcal{A}, s_0, H, \{\hat p _h \}_{h \in [H]}, \{ \tilde r_{h, \kappa} \}_{h \in [H]} )$ via backward induction. The following result shows that choosing $m$ and $\kappa$ appropriately ensures that \Cref{eq:pac-obj} holds.

\begin{theorem}[Upper Bound]\label{thm:learning}
    Choose $m = \widetilde{\mathcal{O}}\left( SK^2H^2 \log\left(\frac{1}{\delta} \right) + {KH^3\log(\frac{1}{\delta})}{\epsilon^{-2}} \right)$ and $\kappa = \mathcal{O}(\epsilon H^{-1})$. Then, for any time homogeneous MDP $\mathcal{M}$, the proposed algorithm is $(\epsilon, \delta)$-correct and its sample complexity is given by:
    \begin{align}\label{eq:ub-rates}
        \widetilde{\mathcal{O}}\left( A S^2 K^2 H^2 \log\left( \frac{1}{\delta} \right)+ \frac{KSAH^3}{\epsilon^2} \log\left( \frac{1}{\delta} \right) \right).
    \end{align}
    If $\mathcal{M}$ is time inhomogeneous, instead, the sample complexity is $\widetilde{\mathcal{O}}\left( A S^2 K^2 H^3 \log\left( \frac{1}{\delta} \right)+ \frac{KSAH^4}{\epsilon^2} \log\left( \frac{1}{\delta} \right) \right)$.
\end{theorem}

\paragraph{Results.}
We start by noticing that, in the small-accuracy regime, the dominant term in \Cref{eq:ub-rates} is $\widetilde{\mathcal{O}}(KSAH^3/\epsilon^2)$. Therefore, for sufficiently small $\epsilon$, our algorithm matches the lower bound in \Cref{thm:lb}, thereby establishing tight rates for the \maxatk problem. In particular, we have proven that the linear dependency on $K$ is the correct scaling for our setting. As a takeaway, we can conclude that when using a small budget $K$, the number of samples required to achieve $\epsilon$-optimal behavior is still comparable to that needed for standard RL.
The algorithm is also computationally efficient since, after a uniform sampling phase, it solves the empirical discretized problem by dynamic programming on $\widehat{\mathcal{M}}_\kappa$. Due to \Cref{theo:approximation}, this procedure is efficient. Finally, as is usual in PAC learning for MDPs \citep[see, e.g.,][]{li2020breaking}, an additional $ H$-factor is incurred in the sample complexity when the MDP is time-inhomogeneous.

\paragraph{Technical Remarks.} To obtain \Cref{thm:learning}, we first account for the error introduced by planning with discretized rewards. Using arguments similar to \Cref{theo:approximation} (\ie small changes in rewards lead to small changes in value functions), this only introduces an additive error term of $\mathcal{O}(H\kappa)$ which we balance by choosing $\kappa$ as  $\mathcal{O}(\epsilon H^{-1})$. The remaining analysis is left to control the planning error that arises from planning with an empirical model. In other words, we need to study how errors propagate within the Bellman equations introduced in \Cref{sec:approx-planning}. Here, we note that directly applying existing finite-horizon guarantees on the augmented MDP would yield suboptimal dependence on $K$ and $H$, since the effective horizon is $KH$. We therefore establish a refined analysis for the simulation error for planning with the empirical model. The techniques behind this are inspired by existing analyses that yield tight rates for classical RL \citep[e.g.,][]{gheshlaghi2013minimax}, but they are tailored to exploit the structure of the \maxatk objective to obtain the desired scale in $K$ and $H$.

%% file: imgs/mdp_hard.tex
\begingroup
\thinmuskip=0mu
\medmuskip=0mu
\thickmuskip=0mu

\tikzset{every picture/.style={line width=0.75pt}}

\begin{tikzpicture}[
    >=Latex,
    state/.style={
        circle,
        draw=black!75,
        fill=white,
        minimum size=7mm,
        inner sep=0pt,
        font=\small
    },
    aone/.style={
        ->,
        draw=blue!75!black,
        line width=0.7pt,
        dashed
    },
    atwo/.style={
        ->,
        draw=red!75!black,
        line width=0.7pt,
        dashed
    },
    lab/.style={
        font=\small,
        fill=white,
        inner sep=1pt
    }
]

\node[state] (s) at (0,0) {$s$};
\node[state] (one) at (-1.8,-2.3) {$1$};
\node[state] (zero) at (1.8,-2.3) {$0$};

\node[blue!75!black] at (-2.25,-0.5) {$a_1$};
\node[red!75!black] at (2.25,-0.5) {$a_2$};

\draw[aone] (s) to[bend right=25] node[lab] {$p$} (one);
\draw[aone] (s) to[bend left=25] node[lab] {$1-p$} (zero);

\draw[atwo] (s) to[bend left=12] node[lab] {$p$} (one);
\draw[atwo] (s) to[bend right=12] node[lab] {$1-p$} (zero);

\end{tikzpicture}
\endgroup

%% file: src/conclusion.tex
\section*{Acknowledgments}
This work was partially funded by the European Union. Views and opinions expressed are however those of the author(s) only and do not necessarily reflect those of the European Union or the European Research Council Executive Agency. Neither the European Union nor the granting authority can be held responsible for them.
This work is supported by an ERC grant (Project 101165466 — PLA-STEER).

%% file: src/appendix_planning.tex
\section{Proofs of \Cref{sec:planning}: Planning}

\subsection{Proof of \Cref{prop:comp_LB}}

\begin{proof}[Proof of \Cref{prop:comp_LB}]
    Consider any $K\in\mathbb{N}$, and an MDP with a single state and two actions $a_s$ and $a_r$. Moreover, let $\mu>1/K$ be a constant defined in the following.
    $a_s$ leads to a reward of $1/K$ with probability $1/K$ (and $0$ otherwise), while $a_r$ leads to a reward of $\mu$ with probability $1/K^2$ and $0$ otherwise. Consider the following history dependent policy that plays $a_s$ until it observes the first success, then plays $a_r$. Let $W$ be the number of turns spent by playing $a_s$ by the policy. Conditioning on $W=w\le K$, the expected reward of this policy is $\tfrac{1}{K}(1-\tfrac1{K^2})^{K-w}+\mu(1-(1-\tfrac1{K^2})^{K-w})$ and thus
    \begin{align*}
    \max_{\pi \in \histPol} V_{1,1}^\pi(s) & \ge \sum_{w=1}^K\tfrac1K(1-\tfrac1K)^{w-1}\Big[\tfrac{1}{K}(1-\tfrac1{K^2})^{K-w}+\mu\big(1-(1-\tfrac1{K^2})^{K-w}\big)\Big]\\
    & = \tfrac1{K^2}\sum_{w=1}^K(1-\tfrac1K)^{w-1}(1-\tfrac1{K^2})^{K-w}+\tfrac{\mu}{K}\sum_{w=1}^K(1-\tfrac1K)^{w-1}-\tfrac{\mu}{K}\sum_{w=1}^K(1-\tfrac1K)^{w-1}(1-\tfrac1{K^2})^{K-w}\\
    &=\tfrac{1}{K}(\tfrac1K-\mu)\left[\sum_{w=1}^K(1-\tfrac1K)^{w-1}(1-\tfrac1{K^2})^{K-w}\right]+\tfrac{\mu}{K}\sum_{w=1}^K(1-\tfrac1K)^{w-1}\\
    &=\tfrac{1}{K}(\tfrac1K-\mu)\frac{\beta^K-\alpha^K}{\beta-\alpha}+\frac\mu K\frac{1-\alpha^K}{1-\alpha}\tag{$\alpha=1-\tfrac1K$, $\beta=1-\tfrac1{K^2}$}\\
    &=(\tfrac1K-\mu)\frac K{K-1}(\beta^K-\alpha^K)+\mu (1-\alpha^K)\\
    &=\frac1{K-1}(\beta^K-\alpha^K)+\mu\big[1-\tfrac K{K-1}\beta^K+\tfrac{1}{K-1}\alpha^K\big].
    \end{align*}

    Consider now the first term. Clearly $\beta^K\ge 1-1/K$ and $\alpha^K\le e^{-1}$ and thus the first term can be lower bounded by $(1-1/e)/K-1/K^2$. For the second term, instead, we can upper bound $\beta^K\le 1-\tfrac1K+\binom{K}{2}\tfrac1{K^4}$ by the binomial expansion and $\tfrac{1}{K-1}\alpha^K\ge \tfrac{1}{eK}$ and thus, simplifying, one can lower bound the second term with $\tfrac1{eK}-\tfrac1{2K^2}$. Thus, we obtained that
    \[
    \max_{\pi \in \histPol} V_{1,1}^\pi(s) \ge \frac{1-e^{-1}(1-\mu)}{K}-\tfrac{1}{K^2}(1+\tfrac\mu2).
    \]

    Conversely, consider any deterministic Markovian policy.\footnote{Recall that the objective is multi-linear in the policy and hence deterministic policies attains the max in $\markPol$. Therefore, we can focus on deterministic policies and ignore the stochastic ones.} 
    It is parametrized by the fixed number $W$ of times it will play the safe action $a_s$ (and $K-W$ times it will thus play the risky action $a_r$). The value is then
    \begin{align*}
    \max_{\pi \in \markPol} V_{1,1}^\pi(s) &\le \max_{W\in[K]} \Big[\mu(1-\beta^{K-W})+\tfrac1K\beta^{K-W}(1-\alpha^{W})\Big]\\
    &\le \max_{x\in[0,1]}\Big[\tfrac\mu k (1-x)+\tfrac1k (1-\alpha^{xK})\Big]\tag{$x:=W/k$ and $1-\beta^{xK}\le (1-x)/K$}\\
    &\le \max_{x\in[0,1]}\Big[\tfrac\mu K (1-x)+\tfrac1k (1-e^{-x}+K^{-1})\Big],
    \end{align*}
    which is maximized for $x=\ln(1/\mu)$ and gives
    \[
    \max_{\pi \in \markPol} V_{1,1}^\pi(s) \le \tfrac\mu K(1-\ln(\mu^{-1}))+\tfrac1K(1-\mu+K^{-1})=\tfrac1K(1+\tfrac1{K}+\mu\ln(\mu)).%
    \]
    Finally, for any $K\ge 2$ and $\mu=1/2$:
    \[
    \frac{\max_{\pi \in \histPol} V_{1,1}^\pi(s)}{\max_{\pi \in \markPol} V_{1,1}^\pi(s)}\ge \frac{1-e^{-1}(1-\mu)}{1+\mu\ln(\mu)}-O(K^{-1})\ge 1.24-O(K^{-1}),
    \]
    concluding the proof.
\end{proof}

\subsection{Proof of \Cref{prop:sufficient-condition-for-optimality}}

Here, we prove that restricting the attention to policies in $\compPol$ is sufficient for optimality.

\begin{proof}[Proof of \Cref{prop:sufficient-condition-for-optimality}]
    We prove the result by construction. Fix any starting $s \in \mathcal{S}$.
    Specifically, we show that there exists a deterministic policy $\mu^\star \in \compPol$ such that for any history $\mathcal{H}_{h,i}$ the following holds: 
    \[
        V_{h,i}^\star(\mathcal{H}_{h,i}) - V^{\mu^\star}_{h,i}(\mathcal{H}_{h,i}) = 0.
    \]
    To this end, we first show that value functions of policies in $\compPol$ depend on the full history only through the compressed history. Precisely, fix $\mu \in \compPol$, and consider two histories $\mathcal{H}_{h,i}^{(1)}$ and $\mathcal{H}_{h,i}^{(2)}$ such that $\Phi(\mathcal{H}_{h,i}^{(1)}) = \Phi(\mathcal{H}_{h,i}^{(2)})$. Then, we prove by induction that $V^\mu_{h,i}(\mathcal{H}_{h,i}^{(1)}) = V^\mu_{h,i}(\mathcal{H}_{h,i}^{(2)})$. As a base case, consider rollout $K$ and step $H+1$. Since the two histories have the same compressed history, they have the same $M_K$ and the same $G_{H+1,K}$. Hence,
    \[
    V^\mu_{H+1,K}(\mathcal{H}_{H+1,K}^{(1)}) = \max\{M_K, G_{H+1,K} \}
    =V^\mu_{H+1,K}(\mathcal{H}_{H+1,K}^{(2)}).
    \]
    Then, suppose that the claim holds for histories at rollout $i$ and step $h+1$, \ie that for all $\mathcal{H}^{(1)}_{h+1,i}, \mathcal{H}^{(2)}_{h+1,i}$ such that $\Phi(\mathcal{H}^{(1)}_{h+1,i}) = \Phi(\mathcal{H}^{(2)}_{h+1,i})$ we have that $V^\mu_{h+1,i}(\mathcal{H}_{h+1,i}^{(1)}) = V^\mu_{h+1,i}(\mathcal{H}_{h+1,i}^{(2)})$. We prove that the result holds for histories at rollout $i$ and step $h \in [H]$. First, observe that, since $\Phi(\mathcal{H}_{h,i}^{(1)}) = \Phi(\mathcal{H}_{h,i}^{(2)})$, the two histories have the same current state $s_{h,i}$. Furthermore,
    \begin{align*}
        V^\mu_{h,i}(\mathcal{H}_{h,i}^{(1)}) & = \sum_{a \in \mathcal{A}} \mu_{h,i}(a \mid \mathcal{H}_{h,i}^{(1)}) \sum_{s' \in \mathcal{S}} p_h(s' \mid s_{h,i},a) V^\mu_{h+1,i}(\mathcal{H}^{(1)}_{h,i} \circ(a, s')) \\
        &  = \sum_{a \in \mathcal{A}} \mu_{h,i}(a \mid \mathcal{H}_{h,i}^{(2)}) \sum_{s' \in \mathcal{S}} p_h(s' \mid s_{h,i},a) V_{h+1,i}^\mu(\mathcal{H}^{(1)}_{h,i} \circ(a, s')) \tag{$\mu \in \compPol$} \\
        & = \sum_{a \in \mathcal{A}} \mu_{h,i}(a \mid \mathcal{H}_{h,i}^{(2)}) \sum_{s' \in \mathcal{S}} p_h(s' \mid s_{h,i},a) V_{h+1,i}^\mu(\mathcal{H}^{(2)}_{h,i} \circ(a, s')) \\
        & = V^\mu_{h,i}(\mathcal{H}^{(2)}_{h,i}),
    \end{align*}
    where, in the third equality, we have used the inductive hypothesis. Indeed, observe that $\Phi(\mathcal{H}^{(1)}_{h,i}) = \Phi(\mathcal{H}^{(2)}_{h,i})$ implies that
    $\Phi(\mathcal{H}^{(1)}_{h,i} \circ(a, s')) = \Phi(\mathcal{H}^{(2)}_{h,i} \circ(a, s'))$ for any state action pair $(s',a)$. It remains to handle the boundary between two consecutive rollouts, but those are direct since $V^\mu_{H+1,i}(\mathcal{H}_{H+1,i}) = V^\mu_{H+1,i}(\mathcal{H}_{1,i+1})$ by definition. This conclude the inductive argument.

    Now, for any history $\mathcal{H}_{h,i}$, let us define the optimal value over policies in $\compPol$ as follows:
    \[
    \overline V^\star_{h,i}(\mathcal{H}_{h,i})(\mathcal{H}_{h,i}) \coloneqq \sup_{\mu \in \compPol} V^\mu_{h,i}(\mathcal{H}_{h,i}).
    \]
    Since $V^\mu_{h,i}(\mathcal{H}_{h,i})$ is constant over all histories with the same compressed representation, we also have that $\overline V^\star_{h,i}(\mathcal{H}_{h,i})$ is constant among histories with the same compressed representation. Furthermore, since the optimization is only over policies within $\compPol$, we also have that $\overline V^\star_{h,i}$ satisfies Bellman-like optimality conditions.\footnote{To this end, it is important to recall again that $\Phi(\mathcal{H}^{(1)}_{h,i}) = \Phi(\mathcal{H}^{(2)}_{h,i})$ implies that
    $\Phi(\mathcal{H}^{(1)}_{h,i} \circ(a, s')) = \Phi(\mathcal{H}^{(2)}_{h,i} \circ(a, s'))$ for any state action pair $(s',a)$.} Specifically, we have that:
    \[
    \overline V^\star_{h,i}(\mathcal{H}_{h,i})(\mathcal{H}_{h,i}) = \max_{a \in \mathcal{A}} \sum_{s' \in \mathcal{S}} p_h(s' \mid s_{h,i}, a) \overline{V}^\star_{h+1,i}(\mathcal{H}_{h,i} \circ (a, s')).
    \]
    with boundary conditions $\overline{V}^\star_{H+1,i}(\mathcal{H}_{H+1,i}) = \overline{V}^\star_{1,i+1}(\mathcal{H}_{1,i+1})$ for all rollouts $i < K$ and, for the last rollout, $\overline{V}^\star_{H+1,K}(\mathcal{H}_{H+1,K}) = \max_{j\in[K]} \sum_{h=1}^H r_h(s_{h,j}, a_{h,j})$. Then, an optimal deterministic policy $\mu^\star \in \compPol$ can be found by any action that attains the max in the definition of $\overline V^\star$.

    Finally, in the last step of the proof, we show that $\mu^\star$ is optimal in the sense that it attains the same value of the optimal history dependent policy. This is proved with a recursion argument. Specifically, we prove that $\overline{V}^{\star}_{h,i}(\Phi(\mathcal{H}_{h,i})) = V^\star_{h,i}(\mathcal{H}_{h,i})$, from which the claim follows directly. As a base case, consider step $H+1$ at rollout $K$. Observe that
    \[
    \overline{V}^\star_{H+1,K}(\Phi(\mathcal{H}_{H+1,K})) = \max_{j \in [K]} \sum_{h=1}^H r_h(s_{h,j},a_{h,j})
    = V^\star_{H+1,K}(\mathcal{H}_{H+1,K}).
    \]
    Now, the induction has two parts. First, suppose that the claim holds at step $h+1$ in rollout $i$, with $h \in [H]$. Then,
    \begin{align*}
        \overline{V}^\star_{h,i}(\mathcal{H}_{h,i}) & = \max_{a\in\mathcal{A}} \sum_{s'\in\mathcal{S}} p_h(s'\mid s_{h,i},a) \overline{V}^\star_{h+1,i} \left(  \mathcal{H}_{h,i}\circ(s',a)\right) \\
        & = \max_{a \in \mathcal{A}} \sum_{s'\in\mathcal{S}} p_h(s'\mid s_{h,i},a) V^\star_{h+1,i} \left( \mathcal{H}_{h,i}\circ(s',a)\right) \\
        & = V^\star_{h,i}(\mathcal{H}_{h,i}).
    \end{align*}
    Second, suppose that $i<K$ and that the claim holds at the beginning of rollout $i+1$. Then, from the boundary conditions on $\overline{V}^\star$ and $V^\star$, we have that:
    \[
        \overline{V}^\star_{H+1,i}(\mathcal{H}_{H+1,i})
        =
        \overline{V}^\star_{1,i+1}(\mathcal{H}_{1,i+1})
        =
        V^\star_{1,i+1}(\mathcal{H}_{1,i+1})
        =
        V^\star_{H+1,i}(\mathcal{H}_{H+1,i}).
    \]
    Combining the base case with the two induction steps proves that $\overline{V}^{\star}_{h,i}(\mathcal{H}_{h,i}) = V^\star_{h,i}(\mathcal{H}_{h,i})$. This yields the optimality of $\mu^\star$ and, consequently, to the desired result.
\end{proof}

\subsection{Proof of \Cref{prop:sub-optimality}}

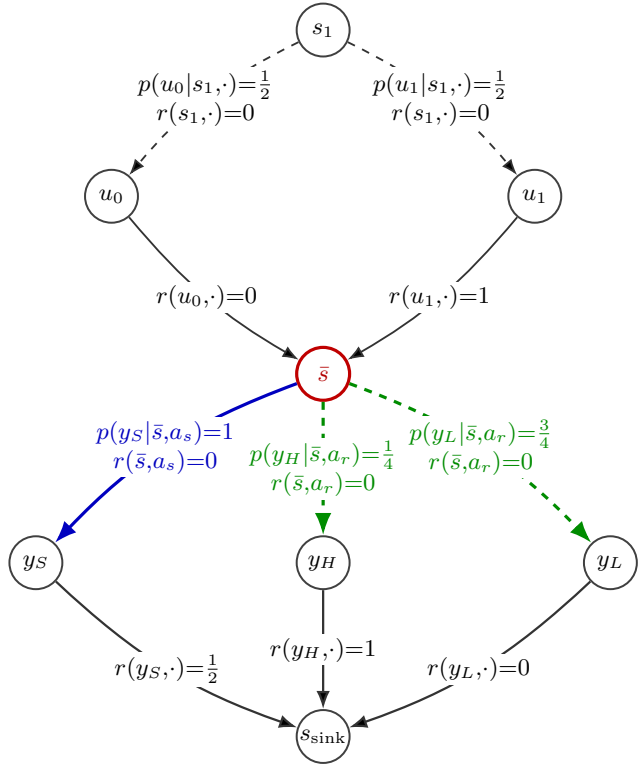
\begin{figure}[ht]
  \centering
  \input{imgs/mdp}
  \caption{MDP that is used as example to prove \Cref{prop:sub-optimality}.}
  \label{fig:mdp}
\end{figure}

\begin{proof}[Proof of \Cref{prop:sub-optimality}]
The proof is by construction. Consider the \maxatk objective with $K=2$ rollouts and the following finite-horizon MDP (see also \Cref{fig:mdp}). There are two actions $\mathcal{A} = \{a_r, a_s \}$, $a_r$ should be interpreted as a risky action, while $a_s$ as a safe one. Each rollout starts in state $s_1$. At the first step both actions yield $0$ reward, and the state transitions to $u_0$ or $u_1$, each with probability $\tfrac12$. Then, at the second step, both actions are equivalent: from $u_0$ the reward is 0, from $u_1$ the reward is 1, and in both cases the process moves deterministically to a state $\bar s$. Once $\bar s$ is reached, the safe action $a_s$ provides a deterministic reward of $\tfrac12$, while the risky action $a_r$ gives reward 1 with probability $\tfrac14$ and reward 0 with probability $\tfrac34$. Therefore, at at state $\bar s$ the cumulative reward accumulated before choosing the last action is either 0 or 1 each with probability $\tfrac12$. Let $G_i$ be the total return in the $i$-the $i$-th rollout, and let $G_i^s$ (resp., $G_i^r$) be the total return in rollout $i$ if the safe (resp., risky) action is selected. Then, we have 
\[
\Prob(G_i^s=\tfrac12)=\tfrac12,\quad \Prob(G_i^s=\tfrac32)=\tfrac12,\quad
\Prob(G_i^r=0)=\tfrac38,\quad
\Prob(G_i^r=1)=\tfrac12,\quad
\Prob(G_i^r=2)=\tfrac18.
\]
We denote by $G=\max\{G_1,G_2\}$ the best return among the two rollouts. 

Let $M$ be the previous best return (\ie the return in the first rollout). We define $U^s(M)$ (resp., $U^r(M)$) as the expected \maxatk value conditional on entering the second rollout with previous best return equal to $M$ and playing the safe (resp., risky) action. Using the return distributions computed above we have that
\[
U^s(M):=\mathbb E[\max\{M,G^s\}]=\frac12 \max\left\{M,\frac12\right\}+\frac12 \max\left\{M,\frac32\right\},
\]
\[
U^r(M):=\mathbb E[\max\{M,G^r\}]=\frac38 M+\frac12 \max\{M,1\}+\frac18 \max\{M,2\}.
\]

We define the optimal continuation value of a rollout as a function of the previous best return $M$ and the current cumulative return $G\in\{0,1\}$ at $\bar s$. In particular, the two last actions are worth
\[
q_s(M,G)=\max\!\Big\{M,\,G+\tfrac12\Big\},
\qquad
q_r(M,G)=\tfrac14\max\{M,\,G+1\}+\tfrac34\max\{M,\,G\}.
\]
Since $G=0$ and $G=1$ each occur with probability $\tfrac12$, the optimal (adaptive) value of a rollout entered with previous best $M$ is
\[
U(M):=\frac12\max\{q_s(M,0),q_r(M,0)\}+\frac12\max\{q_s(M,1),q_r(M,1)\}.
\]
Evaluating $U$ at the only values of $M$ that can arise (a first rollout returns $G_1\in\{\tfrac12,\tfrac32\}$ under $a_s$ and $G_1\in\{0,1,2\}$ under $a_r$) gives
\[
U(0)=1,\quad U\!\left(\frac12\right)=\frac{17}{16},\quad U(1)=\frac54,\quad U\!\left(\frac32\right)=\frac{25}{16},\quad U(2)=2.
\]
We record which action attains the maximum at each of the two branches $G\in\{0,1\}$ through a pair $(a_0,a_1)$ specifying the action played at $G=0$ and $G=1$, respectively. Direct comparison of $q_s$ and $q_r$ yields:
\begin{align*}
M=0:&\quad q_s>q_r \text{ at } G=0 \text{ and } G=1, &&\text{optimal behavior } (a_s,a_s);\\
M=\tfrac12:&\quad q_r>q_s \text{ at } G=0,\ \ q_s>q_r \text{ at } G=1, &&\text{optimal behavior } (a_r,a_s);\\
M=\tfrac32:&\quad q_s=q_r \text{ at } G=0,\ \ q_r>q_s \text{ at } G=1, &&\text{optimal behavior } (a_s,a_r)\text{ or }(a_r,a_r).
\end{align*}
 These optimal behaviors are pairwise incompatible: no single behavior $(a_0,a_1)$, constant across $M$, attains $U(M)$ simultaneously at $M=0$, $M=\tfrac12$, and $M=\tfrac32$.

\paragraph{Lower bound for history-dependent policies.}
Consider the policy that plays $a_s$ throughout the first rollout, so $G_1=\tfrac12$ or $G_1=\tfrac32$ each with probability $\tfrac12$, and then acts greedily in the second rollout. Its value is exactly $\mathbb E[U(G_1)]$, so
\[
\OPT_{\mathrm{adap}}\ \ge\ \frac12\,U\!\left(\frac12\right)+\frac12\,U\!\left(\frac32\right)
=\frac12\cdot\frac{17}{16}+\frac12\cdot\frac{25}{16}=\frac{42}{32}. 
\]

\paragraph{Policies depending on current state, rollout index, and previous best return.} A policy in this class cannot condition on $G$ within the current rollout, so in each rollout it commits to a single last action, independently of $G$. In the second rollout, entered with previous best $M$, its value is 
\[
\max\{U^s(M),U^r(M)\}.
\]
In the first rollout the policy commits to a single action. If it plays $a_s$, then $G_1\in\{\tfrac12,\tfrac32\}$ each with probability $\tfrac12$, and its value is
\[
\frac12\max\left\{S\left(\frac12\right),R\left(\frac12\right)\right\}+\frac12\max\left\{S\left(\frac32\right),R\left(\frac32\right)\right\}
=\frac12\cdot 1+\frac12\cdot\frac{25}{16}=\frac{41}{32}.
\]
If it plays $a_r$, then $G_1\in\{0,1,2\}$ with probabilities $\tfrac38,\tfrac12,\tfrac18$, and its value is
\[
\frac38\max\left\{U^s(0),U^r(0)\right\}+\frac12\max\left\{U^s(1),U^r(1)\right\}+\frac18\max\{U^s(2),U^r(2)\}
=\frac38\cdot 1+\frac12\cdot\frac54+\frac18\cdot 2=\frac{10}{8}.
\]
Let $\OPT_{\mathrm{noG}}$ the optimal value attained by the class of policies that cannot condition on $G$. Taking the best of the two options above we have
\[
\OPT_{\mathrm{noG}}=\max\left\{\frac{41}{32},\frac{10}{8}\right\}=\frac{41}{32}<\OPT_{\mathrm{adap}}.
\]

\paragraph{Policies depending on current state, rollout index, and cumulative return.}
We now consider policies that can condition on the current cumulative return $G$ in the current rollout, but not on the previous best return $M$. At $\bar s$, such a policy selects an action as a function of $G\in\{0,1\}$ only. Hence, in each rollout its behavior at $\bar s$ is described by a pair
\[
(a_0,a_1)\in\{a_s,a_r\}^2,
\]
where $a_g$ is the action selected when the cumulative return at $\bar s$ is $G=g$. We denote the four possible behaviors by
\[
ss=(a_s,a_s),\qquad rr=(a_r,a_r),\qquad sr=(a_s,a_r),\qquad rs=(a_r,a_s).
\]

Fix any such policy. Let $b_1$ be its behavior in the first rollout and $b_2$ its behavior in the second rollout. Since the policy cannot observe the previous best return, $b_2$ is fixed independently of the realized first-rollout return $M=G_1$. Conditional on $M$, the value obtained in the second rollout under $b_2$ is at most $U(M)$, with equality only if the fixed behavior $b_2$ is optimal for that value of $M$. Therefore, the value of the policy is bounded by
\[
\mathbb E\big[U(G_1)\big],
\]
with strict inequality whenever $b_2$ fails to attain $U(M)$ on a value of $M$ that occurs with positive probability under $b_1$.

We first observe that, if the first-rollout behavior is $rr$ or $sr$, then
\[
\mathbb E\big[U(G_1)\big]
=
\frac38 U(0)+\frac12 U(1)+\frac18 U(2)
=
\frac54
\]
for $rr$, and
\[
\mathbb E\big[U(G_1)\big]
=
\frac12 U\!\left(\frac12\right)+\frac38 U(1)+\frac18 U(2)
=
\frac54
\]
for $sr$. In both cases, even the adaptive continuation upper bound is strictly smaller than $42/32$. Hence these first-rollout behaviors cannot achieve value $42/32$.

It remains to consider $b_1=ss$ and $b_1=rs$. If $b_1=ss$, then
\[
G_1\in\left\{\frac12,\frac32\right\}
\]
with positive probability on both values. To attain the upper bound $\mathbb E[U(G_1)]$, the second-rollout behavior would have to be optimal both at $M=\frac12$ and at $M=\frac32$. However, from the comparison above, optimality at $M=\frac12$ requires behavior $(a_r,a_s)$, whereas optimality at $M=\frac32$ requires behavior $(a_s,a_r)$ or $(a_r,a_r)$. No single behavior is optimal for both values of $M$. Therefore,
\[
\mathbb E\!\left[\max\{G_1,G_2\}\mid b_1=ss,b_2\right]
<
\mathbb E\big[U(G_1)\mid b_1=ss\big]
=
\frac12 U\!\left(\frac12\right)+\frac12 U\!\left(\frac32\right)
=
\frac{42}{32}.
\]

Similarly, if $b_1=rs$, then
\[
G_1\in\left\{0,1,\frac32\right\},
\]
and both $0$ and $\frac32$ occur with positive probability. To attain the upper bound $\mathbb E[U(G_1)]$, the second-rollout behavior would have to be optimal both at $M=0$ and at $M=\frac32$. Optimality at $M=0$ requires behavior $(a_s,a_s)$, whereas optimality at $M=\frac32$ requires behavior $(a_s,a_r)$ or $(a_r,a_r)$. Again, no single behavior is optimal for both values. Hence,
\[
\mathbb E\!\left[\max\{G_1,G_2\}\mid b_1=rs,b_2\right]
<
\mathbb E\big[U(G_1)\mid b_1=rs\big]
=
\frac38 U(0)+\frac18 U(1)+\frac12 U\!\left(\frac32\right)
=
\frac{42}{32}.
\]

Combining the four cases for $b_1$, every policy that conditions on the current state, rollout index, and current cumulative return, but not on the previous best return, has value $\OPT_{\mathrm{noM}}$ strictly smaller than $42/32$. Since we already exhibited a history-dependent policy with value at least $42/32$, we obtain
\[
\OPT_{\mathrm{noM}} < \OPT_{\mathrm{adap}}.
\]
This proves the result.
\end{proof}

\subsection{Proof of \Cref{theo:approximation}}

To prove \Cref{theo:approximation} we first derive the following result that studies the difference, in term of performance, of a given history dependent policy $\pi$ which is evaluated on two different MDPs that differs only in term of the reward function.

\begin{lemma}[Value Difference: Same Policy, Different Rewards]\label{lemma:same-pol-diff-rew}
    Consider two MDPs $\mathcal{M}$ and $\widetilde{\mathcal{M}}$ that differs only in the reward function and denote by $\{r_h)\}_{h=1}^{H}$ and $\{\widetilde r_h)\}_{h=1}^{H}$ the reward functions of $\mathcal{M}$ and $\widetilde{\mathcal{M}}$ respectively. Consider any history dependent policy $\pi \in \histPol$. Then, for any rollout $i \in [K]$, step $h$, and history $\mathcal{H}_{h,i}$ it holds that:
    \begin{align*}
        \Big|V^\pi_{h,i}(\mathcal{H}_{h,i}) - \widetilde{V}^\pi_{h,i}(\mathcal{H}_{h,i})\Big| \le H \| r- \widetilde r  \|_{\infty},
    \end{align*}
    where $V^\mu_{h,i}(\cdot)$ and $\widetilde V^\mu_{h,i}(\cdot)$ denotes the value functions of policy $\pi$ in the MDPs $\mathcal{M}$ and $\widetilde{\mathcal{M}}$ respectively.
\end{lemma}
\begin{proof}
    We prove this result by induction. As a base case, consider step $H+1$ at rollout $K$. Here, for any history $\mathcal{H}_{H+1, K}$ it holds that:
    \begin{align*}
    \Big|V^\pi_{H+1,K}(\mathcal{H}_{H+1,K}) - \widetilde{V}^\pi_{H+1,K}(\mathcal{H}_{H+1,K}) \Big| & = \Big| \max_{j \in [K]} \sum_{h \in [H]} r_h(s_{h,j}, a_{h,j}) - \max_{j \in [K]} \sum_{h \in [H]} \widetilde r_h(s_{h,j}, a_{h,j}) \Big| \\
    & \le \sum_{h \in [H]} \max_{s,a,h} |r_{h}(s,a) - \widetilde{r}_h(s,a)| \\
    & = H \| r - \widetilde r  \|_{\infty} 
    \end{align*}
    Next, suppose that the claims holds at rollout $i$ and step $h+1$. Then, we have that:
    \begin{align*}
        \Big|V^\pi_{h,i}(\mathcal{H}_{h,i}) - \widetilde{V}^\pi_{h,i}(\mathcal{H}_{h,i}) \Big| & = \Big| \sum_{a \in \mathcal{A}} \pi_{h,i}(a \mid \mathcal{H}_{h,i}) \sum_{s' \in \mathcal{S}} p_h(s'\mid s,a) \left( V^\pi_{h+1,i}(\mathcal{H}_{h,i} \circ(a, s') ) - \widetilde{V}^\pi_{h+1,i}(\mathcal{H}_{h,i} \circ(a, s')) \right)  \Big|  \\
        & \le H \| r - \widetilde r  \|_{\infty} \Big| \sum_{a \in \mathcal{A}} \pi_{h,i}(a \mid \mathcal{H}_{h,i}) \sum_{s' \in \mathcal{S}} p_h(s'\mid s,a)  \Big| \tag{Inductive hp.} \\
        & = H \| r - \widetilde r  \|_{\infty}.
    \end{align*}
    Finally, suppose that $i < K$ and that the claims holds at the beginning of rollout $i+1$. Then, from the boundary conditions of the value function, we have that:
    \begin{align*}
        \Big|V^\pi_{h,i}(\mathcal{H}_{H+1,i}) - \widetilde{V}^\pi_{H+1,i}(\mathcal{H}_{H+1,i}) \Big| = \Big|V^\pi_{1,i+1}(\mathcal{H}_{1,i+1}) - \widetilde{V}^\pi_{1,i+1}(\mathcal{H}_{1,i+1}) \Big| \le H \| r - \widetilde r \|_{\infty}.
    \end{align*}
    This concludes the proof.
\end{proof}

We now prove \Cref{theo:approximation}.

\begin{proof}[Proof of \Cref{theo:approximation}]
    We first prove \Cref{eq:approximation-eq-main} by induction. To formally prove it, we need to introduce some auxiliary notation. Fix $\kappa > 0$; then, for any history dependent policy $\pi \in \histPol$, we write $V^{\pi, \kappa}$ for the value function of policy $\pi$ in $\mathcal{M}_\kappa$.
    Now, let $\mu^\star$ be the optimal policy in $\compPol$ for $\mathcal{M}$, and let $\tilde \mu^\star_\kappa \in \compPol^\kappa$ be the optimal policy for $\mathcal{M}_\kappa$. It holds that:
    \begin{align*}
        V^{\mu^\star}_{h,i}(\mathcal{H}_{h,i}) - V_{h,i}^{\tilde \mu_\kappa^\star}( \mathcal{H}_{h,i}) & = V^{\mu^\star}_{h,i}(\mathcal{H}_{h,i}) - V^{\tilde \mu^\star_\kappa, \kappa}_{h,i}(\mathcal{H}_{h,i}) + V^{\tilde \mu^\star_\kappa, \kappa}_{h,i}(\mathcal{H}_{h,i}) - V_{h,i}^{\tilde \mu_\kappa^\star} (\mathcal{H}_{h,i}) \\
        & \le V^{\mu^\star}_{h,i}(\mathcal{H}_{h,i}) - V^{\mu^\star, \kappa}_{h,i}(\mathcal{H}_{h,i}) + V^{\tilde \mu^\star_\kappa, \kappa}_{h,i}(\mathcal{H}_{h,i}) - V_{h,i}^{\tilde \mu_\kappa^\star}(\mathcal{H}_{h,i}) \\
        & \le 2 H \|r - \tilde r \|_{\infty} \tag{\Cref{lemma:same-pol-diff-rew}} \\
        & \le 2 H \kappa, \tag{Def. of $\kappa$}
    \end{align*}
    where the first inequality follows from the fact that $\tilde \mu^\star_\kappa$ is an optimal policy for for $\mathcal{M}_\kappa$. This concludes the proof of \Cref{eq:approximation-eq-main}.

    Finally, we conclude by analyzing the computational complexity of solving $\mathcal{M}_{\kappa}$. At the first step of the backward induction, the compressed and discretized history representations are $|\mathcal{B}_\kappa|$ and $|\mathcal{C}_{H+1, \kappa}|$, where $\mathcal{B}_\kappa$ is the set of possible previous best returns in the discretized MDP $\mathcal{M}_\kappa$, and $\mathcal{C}_{h,\kappa}$ is the set of possible partial returns at step $h \in [H+1]$ in $\mathcal{M}_\kappa$. More generally, when solving the maximum over actions in the backward induction at step $h$ and rollout $i$, the set of discretized compressed histories has size $S \cdot |\mathcal{B}_\kappa| \cdot |\mathcal{C}_{h,\kappa}|$. From this argument, it follows that the complexity of tabular backward induction is of order $\mathcal{O}\left(K A S^2 |\mathcal{B}_\kappa| \sum_{h=1}^{H+1} |\mathcal{C}_{h,\kappa}| \right)$. Indeed, for each of the $K$ rollouts, at step $h$ there are $|C_{h,\kappa}| \cdot |\mathcal{B}_\kappa| \cdot S$ compressed history representations, and the per-step complexity of this method is $SA$. Now, observe that both $|\mathcal{B}_\kappa|$ and $|C_{h,\kappa}|$ can be bounded by $1 + \tfrac{H}{\kappa}$. This yields the result.
\end{proof}

\subsection{Proof of \Cref{th:NP-hard}}

\begin{proof}[Proof of \Cref{th:NP-hard}]
    We are given a \textsc{Subset-Sum} instance $S=w_1,\ldots, w_n$, and target value $L$. 
    We define a few parameters first. Let $\eta=\frac{2^{-n}}8$, a normalization constant $C=L+5+\sum_{j=1}^{n}w_j$, and a threshold $\beta=\frac{2^{-n}}{16 C}$.
    
    \paragraph{Construction.} We build the following MDP.
    We consider only $K=2$ rollouts. We start from the state $s_0$, in which we have three actions $a^\textsf{safe}$, $a^+$, and $a^-$. 

    \begin{itemize}
    \item Playing $a^\textsf{safe}$ leads deterministically to a reward of $(1+L)/C$ and the state transition deterministically to an absorbing state.
    
    \item Playing $a^+$ leads to an immediate return of $(1+\eta)/C$ after which the state transition to the state $s_1$. 
    
    \item Action $a^-$ leads to an immediate reward of $0$ with probability $1/2$ and of $2/C$ with probability $1/2$, and in both cases it transitions to the state $s_1$; 
    
    \item At each state $s_i$ with $i\in\{1,\ldots, n\}$, there is only a dummy action $a_0$, which collects with probability $1/2$ a reward of $w_i/C$ and $0$ otherwise, and transits deterministically to the state $s_{i+1}$. State $s_n$, with probability $1/2$ collects a reward of $(L+3)/C$ and $0$ otherwise, after which the game ends.
    \end{itemize}

    Consider the random variable: $X=\sum_{j=1}^nw_j Z_j+(L+3)Z_0$, where $Z_0,Z_1,\ldots, Z_n$ are i.i.d.~Bernoulli of parameter $1/2$. There are essentially only three actions in the MDP. Playing $a^\textsf{safe}$ leads to the deterministic reward of $R^\textsf{safe}=(1+L)/C$. Playing $a^+$ leads to a stochastic reward of $R^+=(1+\eta+X)/C$, while playing $a^-$ leads to a stochastic reward of $R^-=(X+2Z)/C$, where $Z$ is a Bernoulli of parameter $1/2$.%

    \paragraph{Useful Claims.} The main property that we need in this reduction is that $\mathbb{P}(X=L)>0$ if and only if the \textsc{Subset-Sum} instance is in the language, i.e., if there exists $S^*\subseteq S$ such that $\sum_{j\in S^*} w_j=L$. In particular, if there exists such a set $S^*$ we get that
    \begin{align}\label{eq:NP_Tmp1}
    \mathbb{P}[X=L]\ge \big(\frac{1}{2}\big)^{n+1}.
    \end{align}
    Then define $Q^+$ ($Q^-$ respectively), the expected values of playing $a^\textsf{safe}$ in the first rollout and $a^+$ ($a^-$ respectively) in the second rollout. Formally:
    \[
    Q^+=\mathbb{E}[\max(R^+-R^\textsf{safe},0)]\quad\text{and}\quad Q^-=\mathbb{E}[\max(R^--R^\textsf{safe},0)].
    \]
    Moreover, notice that
    \begin{align*}
        Q^+&=\frac 1C\mathbb{E}[\max(\eta+X-L,0)]\\
        &=\frac 1C\mathbb{E}[\max(X-L,0)]+\frac 1C\eta\mathbb{P}[X\ge L].
    \end{align*}
    where the last equality follows from $\eta<1$ and $X,L$ are natural numbers.
    Similarly, one can verify that:
    \begin{align*}
        Q^+&=\frac 1C\mathbb{E}[\max(X-L+2Z-1,0)]\\
        &=\frac 1C\mathbb{E}[\max(X-L,0)]+\frac{1}{2C}\mathbb{P}[X=L].
    \end{align*}

    Notably, we can conclude that 
    \(
    Q^--Q^+=\frac{1}{C}\left(\frac12\mathbb{P}[X=L]-\eta\mathbb{P}[X\ge T]\right).
    \)

    \paragraph{Completeness.}
    Assume that there exists a set $S^*\subseteq S$ such that $\sum_{j\in S^*} w_j=L$. Assume we played $a^\textsf{safe}$ in the first rollout. Now we are given the history of the first rollout, and we have to determine the best action to play. Playing either $a^+$ or $a^-$ is optimal because, with some probability, we will get a large payoff of $(T+3)/C$ in state $s_n$, which will dominate the reward of $R^\textsf{safe}=(L+1)/C$, collected in the first rollout. Since we are in a Yes instance, we know that:
    \begin{align*}
    Q^--Q^+&=\frac{1}{C}(\frac12\mathbb{P}[X=L]-\eta\mathbb{P}[X\ge L])\\
    &\ge \frac1C\Big(\frac12\big(\frac12\big)^{n+1}-\eta\Big)\tag{\Cref{eq:NP_Tmp1} and $\mathbb{P}[X\ge L]\le 1$}\\
    &=\frac1C\frac{2^{-n}}{8}\ge \beta,
    \end{align*}
    showing that playing $a^-$ is better with a gap of $\beta$ than playing $a^+$ in this history.

    \paragraph{Soundness.}

    Consider the case in which there is no such subset $S^*$. Then $\mathbb{P}(X=L)=0$. However, we can always get more than $(L+1)/C$ since at state $s_n$ we can receive a reward of $(L+3)/C$ with probability $1/2$, i.e., $\mathbb{P}(X>L)\ge 1/2$. Hence
    \[
    Q^--Q^+=-\frac{\eta}{C}\mathbb{P}[X\ge L]\le -\frac{\eta}{C}\frac{1}{2}=-\frac{2^{-n}}{16C}=-\beta.
    \]
    showing that playing $a^-$ is worst with a gap of $\beta$ than playing $a^+$ in this history.
\end{proof}

%% file: imgs/mdp.tex
\begingroup
\thinmuskip=0mu
\medmuskip=0mu
\thickmuskip=0mu

\tikzset{every picture/.style={line width=0.75pt}}

\begin{tikzpicture}[
    node distance=3.1cm,
    >=Latex,
    state/.style={circle,draw=black!75,fill=white,minimum size=7mm,inner sep=0pt,align=center,font=\small},
    redstate/.style={state,draw=red!75!black,very thick,text=red!70!black,minimum size=7mm},
    sinkstate/.style={circle,draw=black!75,fill=white,minimum size=7mm,inner sep=0pt,align=center,font=\small},
    forced/.style={->,double,draw=black!80,line width=0.7pt,double distance=1.4pt},
    blueaction/.style={blue!75!black,->,draw=blue!75!black,very thick},
    greenaction/.style={green!55!black,->,draw=green!55!black,very thick,dashed},
    terminal/.style={->,draw=black!80,line width=0.85pt},
    lab/.style={font=\small,align=center,fill=white,inner sep=1.2pt},
    smalllab/.style={font=\scriptsize,align=center,fill=white,inner sep=1.1pt}
]

\node[state] (s0) at (0,0) {$s_1$};

\node[state] (u0) at (-2.8,-2.2) {$u_0$};
\node[state] (u1) at (2.8,-2.2) {$u_1$};

\node[redstate] (x) at (0,-4.55) {$\bar s$};

\node[state] (yS) at (-3.8,-7.05) {$y_S$};
\node[state] (yH) at (0,-7.05) {$y_H$};
\node[state] (yL) at (3.8,-7.05) {$y_L$};

\node[sinkstate] (sink) at (0,-9.35) {$s_{\mathrm{sink}}$};

\draw[forced,dashed] (s0) edge[bend right=12] node[lab] {$p(u_0|s_1, \cdot) = \tfrac12$\\$r(s_1,\cdot)=0$} (u0);
\draw[forced,dashed] (s0) edge[bend left=12] node[lab] {$p(u_1|s_1, \cdot) = \tfrac12$\\$r(s_1,\cdot)=0$} (u1);

\draw[forced] (u0) edge[bend right=10] node[lab] {$r(u_0,\cdot)=0$} (x);
\draw[forced] (u1) edge[bend left=10] node[lab] {$r(u_1,\cdot)=1$} (x);

\draw[blueaction] (x) edge[bend right=13] node[lab] {$p(y_S|\bar s, a_s) = 1$\\$r(\bar s,a_s)=0$} (yS);
\draw[greenaction] (x) edge node[lab] {$p(y_H|\bar s, a_r) = \tfrac14$\\$r(\bar s,a_r)=0$} (yH);
\draw[greenaction] (x) edge[bend left=13] node[lab] {$p(y_L|\bar s, a_r) = \tfrac34$\\$r(\bar s,a_r)=0$} (yL);

\draw[terminal] (yS) edge[bend right=12] node[lab] {$r(y_S,\cdot)=\frac12$} (sink);
\draw[terminal] (yH) edge node[lab] {$r(y_H,\cdot)=1$} (sink);
\draw[terminal] (yL) edge[bend left=12] node[lab,] {$r(y_L,\cdot)=0$} (sink);

\end{tikzpicture}
\endgroup

%% file: src/appendix_learning_old.tex
\section{Proofs of \Cref{sec:learning}: Learning}\label{app:learning}

\subsection{Proof of \Cref{thm:lb}}

\begin{proof}[Proof of \Cref{thm:lb}]
The proof of \Cref{thm:lb} is articulated in four steps.
First, we present the family of instances that we consider to derive the results. In words, we introduce a base instance, in which all actions are identical, and then, for each state-action pair, we introduce an alternative instance where that action is slightly better for that state. In the second step, we show that in any alternative instance any $\epsilon$-optimal policy should put sufficient mass on the optimal action for that instance. Finally, in the third step, we use this result to define a good event, which we use in the last step (in combination with change-of-measure arguments) to lower bound the expected stopping time of any $(\epsilon, \delta)$-correct algorithm.

\paragraph{Step 1: Instances Construction}

The base instance $\mathcal{M}_0 = (\mathcal{S}, \mathcal{A}, p^{(0)}, r, H)$ is as follows. The set of states is $\mathcal{S} = \{g\} \cup \overline{\mathcal{S}}$, where $g$ is a rewarding and absorbing goal state, and $\overline{\mathcal{S}} = \{s_1, \dots, s_n \}$ is a set of states from which we can try to reach the goal. For the goal state, we have $p^{(0)}(g \mid g,a) = 1$, and $r(g,a) = 1$ for all actions $\forall a \in \mathcal{A}$. For any other state $s_i \in \overline{\mathcal{S}}$ and action $a \in \mathcal{A}$, instead, we have $r(s_i, a) = 1$ and $p^{(0)}(g \mid s_i, a) = p$ and $p^{(0)}(s_i \mid s_i, a) = 1-p$, where $p = c_0 \tfrac{1}{KH}$ for some positive constant $c_0 \in \Reals_{> 0}$ which is defined below. Next, for each state-action pair $(s,a) \in \overline{\mathcal{S}} \times \mathcal{A}$, we introduce an alternative instance $\mathcal{M}_{s,a} = (\mathcal{S}, \mathcal{A}, p^{(s,a)}, r, H)$ which is equal to $\mathcal{M}_0$ everywhere but in $(s,a)$. In particular, $p^{(s,a)}(g\mid s,a) = p + \Delta$ and $p^{(s,a)}(s \mid s,a) = 1-p-\Delta$, where $\Delta = c_1 \frac{\epsilon}{KH^2}$ and $c_1 \in \Reals_{> 0}$ is a positive constant that is defined below.  In the rest of the proof, we denote by $\mathfrak{M} = \mathcal{M}_0 \cup \{ \mathcal{M}_{s,a} \}_{s \in \overline{\mathcal{S}}, a \in \mathcal{A}}$ the family of instances that we are considering.

We choose $c_0 = 1/96$ and $c_1 = 8192$. Under the assumption that $\epsilon \le \tfrac{H}{96 \cdot 8196}$, we have that $\Delta \le \tfrac{p}{2}$. Furthermore, we also observe that $p \le \tfrac14$ for any choice of $K,H$ so that all transitions are correctly defined. In the rest of the proof, we also make the assumption that $H$ is divisible by $4$ to simplify some computations. Finally, we suppose that $A \ge 2$ and $\delta \le \tfrac{1}{16}$.

\paragraph{Step 2: Value Function Analysis} 
Fix any $\mathcal{M}_{s,a} \in \mathfrak{M}$ and let $\pi \in \compPol$ and let $L = \tfrac{H}{4}$. Since $\mathcal{M}_{s,a}$ is fixed within step, we drop the dependency on $\mathcal{M}_{s,a}$ in symbols such as $\E$ and $\Prob$. With some algebraic manipulations, we have that:
\begin{align*}
    V^\star_{1,1}(s) - V^\pi_{1,1}(s) & = V^\star_{1,1}(s) - \E_{\pi}\left[ V^\pi_{H,K}(\mathcal{H}_{H,K}) \right] \\
    & = V^\star_{1,1}(s) - \E_{\pi}\left[ \max_{i \in [K]} \sum_{h \in [H]} r(s_{h,i}, a_{h,i}) \mid \mathcal{H}_{H,K} \right]  \\
    & = \sum_{i \in [K]} \sum_{h \in [H]} \E_\pi\left[ V^*_{h,i}(\mathcal{H}_{h,i}) - V^\star_{h^+, i^+}(\mathcal{H}_{h^+, i^+}) \right] \tag{Telescoping argument}\\
    & \ge \sum_{i \in [K]} \sum_{h \in [L]} \E_\pi\left[ V^\star_{h,i}(\mathcal{H}_{h,i}) - V^\star_{h^+, i^+}(\mathcal{H}_{h^+, i^+}) \right],
\end{align*} 
where the last step follows from the fact that $\E_\pi[V^\star_{h,i}(\mathcal{H}_{h,i})] \ge \E_{\pi}[V^\star_{h^+,i^+}(\mathcal{H}_{h^+,i^+})]$. Indeed, 
\begin{align*}
    \E_\pi\left[ V^\star_{h,i}(\mathcal{H}_{h,i}) \right] & = \E_\pi \left[ \max_{a \in \mathcal{A}} \sum_{s' \in \mathcal{S}} p(s' \mid s,a) V^\star_{h^+, i^*}(\mathcal{H}_{h,i} \circ (s',a)) \right] \\
    & \ge \E_\pi \left[ \sum_{a \in \mathcal{A}} \pi_{h,i}(a \mid \mathcal{H}_{h,i}) \sum_{s' \in \mathcal{S}} p(s' \mid s_{h,i},a) V^\star_{h^+, i^*}(\mathcal{H}_{h,i} \circ (s',a)) \right] \\
    & = \E_\pi[V^\star_{h^+,i^+}(\mathcal{H}_{h^+, i^+})].
\end{align*}
Now, for any pair $(h,i)$, let us denote by $\mathcal{F}_{h,i}$ the event that, for all $(\bar h, \bar i) \prec (h,i)$, $s_{h,i} \ne g$. Then, we have that:
\begin{align}\label{eq:lb-proof-eq-1}
    V^\star_{1,1}(s) - V^\pi_{1,1}(s) & \ge \sum_{i \in [K]} \sum_{h \in [L]} \E_\pi\left[ \left(V^\star_{h,i}(\mathcal{H}_{h,i}) - V^\star_{h^+, i^+}(\mathcal{H}_{h^+, i^+}) \right) \mathds{1}\{ \mathcal{F}_{h,i} \} \right].
\end{align}
We now analyze a single term within the r.h.s of \Cref{eq:lb-proof-eq-1}. To this end, we first make some consideration. First of all, we note that (a) on $\mathcal{F}_{h,i}$ the state in $\mathcal{H}_{h,i}$ is $s_{h,i} = s$ and (b) the optimal policy always play action $a$ in state $s$. Furthermore, (c) we observe that $V^\star(\mathcal{H}_{h,i})$ does not depend on the identity of the actions that have been played throughout the history, but only on the current state, the reward cumulated in the current rollouts, and the best return observed so far. Therefore, in the following, we simplify the notation and write $V_{h^+,i^+}^\star(\mathcal{H}_{h,i} \circ \bar s)$ instead of $V_{h^+,i^+}^\star(\mathcal{H}_{h,i} \circ (\bar s, \bar a))$ for any state-action pair $(\bar s, \bar a) \in \mathcal{S} \times \mathcal{A}$. Finally, since we are considering $\pi \in \compPol$ and since we are interested only in the case where $\mathcal{F}_{h,i}$ holds, we observe that $\pi_{h,i}(\cdot \mid \mathcal{H}_{h,i})$ is uniquely determined by the pair $(h,i)$.\footnote{Indeed, as we noted above, the state is necessarily $s$, the previous best return is $0$, and the cumulative reward in the current rollout is $0$ as well.} Therefore, we will shrink the notation and write $\alpha_{h,i}^{\pi,s,a} \coloneqq \pi_{h,i}(a \mid \mathcal{H}_{h,i} \land \mathcal{F}_{h,i})$. Here, the dependency on $s$ is due to the fact, on $\mathcal{F}_{h,i}$ the state is necessarily $s$ (see point (a) above. To ease the notation, since $(s,a)$ is fixed in this step, we will drop the dependency on $s,a$ within $\alpha_{h,i}^{\pi,s,a}$. At this point, we can proceed by analyzing the r.h.s. of \Cref{eq:lb-proof-eq-1}.
We first focus on $\E_\pi\left[ V^\star_{h,i}(\mathcal{H}_{h,i}) \mathds{1} \{ \mathcal{F}_{h,i} \}\right]$. We have that:
\begin{align*}
    \E_\pi\left[ V^\star_{h,i}(\mathcal{H}_{h,i}) \mathds{1} \{ \mathcal{F}_{h,i} \}\right] = \E_\pi\left[ \left( (p+\Delta) V^\star_{h^+, i^+}(\mathcal{H}_{h,i} \circ g) + (1-p-\Delta)V^\star_{h^+, i^+}(\mathcal{H}_{h, i} \circ s) \right) \mathds{1} \{ \mathcal{F}_{h,i} \}\right],
\end{align*}
where the result is a direct application of the definition of value function.
For the second term of the r.h.s in \Cref{eq:lb-proof-eq-1}, namely $\E_\pi\left[  V^\star_{h^+, i^+}(\mathcal{H}_{h^+, i^+}) \mathds{1}\{ \mathcal{F}_{h,i} \} \right]$, instead, we have that:
\begin{align*}
    \E_\pi\left[  V^\star_{h^+, i^+}(\mathcal{H}_{h^+, i^+}) \mathds{1}\{ \mathcal{F}_{h,i} \} \right] & = \E_\pi\left[  \E_\pi \left[ V^\star_{h^+, i^+}(\mathcal{H}_{h^+, i^+}) \mathds{1}\{ \mathcal{F}_{h,i} \} \mid \mathcal{H}_{h,i}  \right] \right] \tag{Tower law} \\
    & =  \E_\pi\left[ \left( (p+\Delta \alpha_{h,i}^\pi) V^\star_{h^+, i^+}(\mathcal{H}_{h,i} \circ g) + (1-p-\Delta \alpha_{h,i}^\pi) V^\star_{h^+, i^+}(\mathcal{H}_{h,i} \circ s)  \right) \mathds{1}\{\mathcal{F}_{h,i} \} \right].
\end{align*}
where, in the second step, we have exploited the definition of $\mathcal{M}_{s,a}$ together with the fact that the value function of $V^\star$ does not depend on the identity of the previously played actions but only on the state, the value of the current rollout, and the best return collected so far. Plugging in these results within \Cref{eq:lb-proof-eq-1}, we have that:
\begin{align}\label{eq:lb-proof-eq-2}
    V^\star_{1,1}(s) - V^\pi_{1,1}(s) & \ge \sum_{i \in [K]} \sum_{h \in [L]} \E_\pi\left[ \Delta (1-\alpha_{h,i}^\pi) \left( V^\star_{h^+, i^+}(\mathcal{H}_{h,i} \circ g) - V^\star_{h^+, i^+}(\mathcal{H}_{h,i} \circ s) \right) \mathds{1}\{ \mathcal{F}_{h,i} \} \right].
\end{align}
Now, we lower bound and upper bound $V^\star_{h^+, i^+}(\mathcal{H}_{h,i} \circ g)$ and $V^\star_{h^+, i^+}(\mathcal{H}_{h,i} \circ s)$. For the first of these two terms, we can use that the agent reached the goal sufficiently early since $h \in [L]$. Thus, we have that:
\[
    V^\star_{h^+, i^+}(\mathcal{H}_{h,i} \circ g) \ge H-L \ge \frac{3}{4}H.
\]
For the second term, instead, we exploit the fact that the goal is an hard-to-reach state:
\begin{align*}
    V^\star_{h^+, i^+}(\mathcal{H}_{h,i} \circ s) & \le V^\star_{1,1}(s)  \le H \Prob_{\pi^\star}(  \exists (h,i):  s_{h,i} = g  )  \le KH^2 (p+\Delta) \le \frac{H}{64} 
\end{align*}
where the last step follows from the definition of $p$ and the fact that $c_0 \le 1/96$.
Plugging these results within \Cref{eq:lb-proof-eq-2} and using some probabilistic arguments we obtain:
\begin{align*}
    V^\star_{1,1}(s) - V^\pi_{1,1}(s) & \ge \frac{47\Delta H}{64} \sum_{i \in [K]} \sum_{h \in [L]}  (1-\alpha_{h,i}^\pi) \Prob_\pi\left(  \forall (\bar h, \bar i) \prec (h,i): s_{\bar h, \bar i} \ne g ~\right) \\
    & = \frac{47\Delta H}{64} \sum_{i \in [K]} \sum_{h \in [L]}  (1-\alpha_{h,i}^\pi) \prod_{(\bar h, \bar i) \prec (i,h)} \Prob_\pi\left( s_{\bar h, \bar i} \ne g \mid \mathcal{F}_{\bar h^-, \bar i^-} \right) \tag{Chain rule} \\
    & = \frac{47\Delta H}{64} \sum_{i \in [K]} \sum_{h \in [L]}  (1-\alpha_{h,i}^\pi) \prod_{(\bar h, \bar i) \prec (i,h)} \left( \alpha^\pi_{\bar i \bar h} \left(1-p-\Delta \right) + (1-\alpha^\pi_{\bar h, \bar i} (1-p)  \right)) \\
    & \ge \frac{47\Delta H}{64} \sum_{i \in [K]} \sum_{h \in [L]}  (1-\alpha_{h,i}^\pi) (1-p-\Delta)^{KH} \\
    & \ge \frac{1}{4} \cdot \frac{47\Delta H}{64} \left(KL - \sum_{i \in [K]} \sum_{h \in [L]} \alpha_{h,i}^\pi \right) \tag{Bernoulli's inequality, def. of $p,\Delta$ and  $c_0 \le 1/96$} 
\end{align*}
Therefore, rearranging the terms, we have that, for any $\epsilon$-optimal policy in $\mathcal{M}_{s,a}$ it must hold that: 
\begin{align*}
    \sum_{i \in [K]} \sum_{h \in [L]} \alpha_{h,i}^\pi & \ge KL - \frac{256\epsilon}{47\Delta H} \\
    & = KH \left( \frac{1}{4} - \frac{256}{47 c_1} \right) \tag{Def. of $\Delta$} \\
    & \ge \frac{KL}{2}. \tag{$c_1 \ge 8192$}
\end{align*}

\paragraph{Step 3: Good Event} In the previous step, we have shown that, for any $\mathcal{M}_{s,a}$, an $\epsilon$-optimal policy must satisfy 
$\sum_{i \in [K]} \sum_{h \in [L]} \alpha_{h,i}^{\pi,s,a} \ge \tfrac{KL}{2}$. Therefore, for the policy $\hat \pi$ that is returned by an $(\epsilon,\delta)$-correct algorithm on $\mathcal{M}_{s,a}$ must satisfy this requirement in high-probability. Precisely, for each $(s,a) \in \overline{S} \times A$, we denote by define $\mathcal{E}_{s,a}$ as follows:
\begin{align*}
    \mathcal{E}_{s,a} \coloneqq \left\{ \sum_{i \in [K]} \sum_{h \in [L]} \alpha_{h,i}^{\hat \pi,s,a} \ge \tfrac{KL}{2} \right\}.
\end{align*}
Since the algorithm is $(\epsilon,\delta)$-correct, it must hold that 
\begin{align}\label{eq:lb-proof-eq-3}
    \Prob_{\mathcal{M}_{s,a}}(\mathcal{E}_{s,a}) \ge 1-\delta.
\end{align}

\paragraph{Step 4: Change of Measure} Finally, we combine the previous results with change-of-measure arguments. Let us denote by $N_{s,a}(\tau_{\delta}) = \sum_{t=1}^{\tau_{\delta}} \mathds{1} \{ s_{t} = s, a_t = a \}$ the number of samples gathered by algorithm at state-action pair $(s,a)$ up to time $\tau_{\delta}$. Then, for $\mathcal{M}_{s,a} \in \mathcal{M}$. For any random any random-variable $Z$ that takes value in $[0,1]$, we have that:
\begin{align*}
    \textup{kl}(\E_{\mathcal{M}_0}[Z], \E_{\mathcal{M}_{s,a}}[Z]) & \le \textup{KL}\left( \Prob_{\mathcal{M}_0}, \Prob_{\mathcal{M}_{s,a}} \right) \tag{\Cref{lemma:com-1}}\\
    & \le \E_{\mathcal{M}_0}[N_{s,a}(\tau_{\delta})] \cdot \textup{KL}(p^{(0)}(\cdot \mid s,a), p^{(s,a)}(\cdot \mid s,a)) \\
    & \le \E_{\mathcal{M}_0}[N_{s,a}(\tau_{\delta})] \cdot \frac{\textup{TV}(p^{(0)}(\cdot \mid s,a), p^{(s,a)}(\cdot \mid s,a))^2}{(p+\Delta)} \tag{Reverse Pinsker’s Inequality}\\
    & = \E_{\mathcal{M}_0}[N_{s,a}(\tau_{\delta})] \cdot \frac{8\Delta^2}{p},
\end{align*}
where $\textup{kl}(p,q)$ denotes the KL-divergence between Bernoulli distributions with mean $p$ and $q$ and $\textup{KL}\left( \Prob_{\mathcal{M}_0}, \Prob_{\mathcal{M}_{s,a}} \right)$ denotes the KL divergence between $\Prob_{\mathcal{M}_0}$ and $\Prob_{\mathcal{M}_{s,a}}$. Next, we apply the previous result using $Z = \mathds{1}\{ \mathcal{E}_{s,a}\}$ and we obtain that:
\begin{align*}
    \E_{\mathcal{M}_0} [\tau_{\delta}] & \ge \sum_{s \in \overline{S}} \sum_{a \in \mathcal{A}} \E_{\mathcal{M}_0}[N_{s,a}(\tau_{\delta})] \\
    & \ge \frac{p}{8\Delta^2} \sum_{s \in \overline{S}} \sum_{a \in \mathcal{A}}  \textup{kl} \left( \Prob_{\mathcal{M}_0}(\mathcal{E}_{s,a}), \Prob_{\mathcal{M}_{s,a}}(\mathcal{E}_{s,a})  \right) \\
    & \ge \frac{p}{8\Delta^2} \sum_{s \in \overline{S}} \sum_{a \in \mathcal{A}} \left( 1 - \Prob_{\mathcal{M}_0}\left( \mathcal{E}_{s,a} \right) \right)  \log\left( \frac{1}{1-\Prob_{\mathcal{M}_{s,a}}(\mathcal{E}_{s,a})} \right) - \log(2) \tag{\Cref{lemma:com-2}} \\
    & \ge \frac{p}{2\Delta^2}\left( \overline{S} A \log\left( \frac{1}{\delta} \right) - \log(2) \overline{S} A - \log\left( \frac{1}{\delta} \right) \sum_{s \in \overline{S}} \sum_{a \in \mathcal{A}}  \Prob_{\mathcal{M}_0}\left( \mathcal{E}_{s,a} \right) \right)  \tag{\Cref{eq:lb-proof-eq-3}}.
\end{align*}
At this point, let us analyze in more detail $\sum_{s \in \overline{S}} \sum_{a \in \mathcal{A}} \Prob_{\mathcal{M}_0}(\mathcal{E}_{s,a})$. Observe, that, for any policy $\pi$ and any state $s$, it holds that 
\begin{align}\label{eq:lb-proof-eq-4}
    \sum_{a \in \mathcal{A}} \sum_{i \in [K]} \sum_{h \in [L]} \alpha^{\pi, s,a}_{h,i} = KL.
\end{align}
Therefore, it holds that:
\begin{align*}
    \sum_{s \in \overline{S}} \sum_{a \in \mathcal{A}} \Prob_{\mathcal{M}_0}(\mathcal{E}_{s,a}) & = \sum_{s \in \overline{S}}  \sum_{a \in \mathcal{A}} \E_{\mathcal{M}_0}\left[ \mathds{1} \left\{ \sum_{i \in [K]} \sum_{h \in [L]} \alpha_{h,i}^{\hat \pi,s,a} \ge \tfrac{KL}{2} \right\} \right]   \\
    & = \sum_{s \in \overline{S}} \E_{\mathcal{M}_0} \left[ \sum_{a \in \mathcal{A}} \mathds{1} \left\{ \sum_{i \in [K]} \sum_{h \in [L]} \alpha_{h,i}^{\hat \pi,s,a} \ge \tfrac{KL}{2} \right\} \right] \\
    & \le\overline{S},
\end{align*}
where in the inequality step we have used \Cref{eq:lb-proof-eq-4}. Thus, we obtained:
\begin{align*}
    \E_{\mathcal{M}_0}[\tau_{\delta}] & \ge \frac{p}{2\Delta^2}\left( \overline{S} A \log\left( \frac{1}{\delta} \right) - \log(2) \overline{S} A - \log\left( \frac{1}{\delta} \right) \overline{S} \right) \\
    & \ge \frac{p}{16\Delta^2} S A \log\left( \frac{1}{\delta} \right) \tag{$A, S\ge 2$ and $\delta \le \frac{1}{16}$} \\
    & = \frac{c_0}{8 c_1^2} \cdot  \frac{KH^3AS }{\epsilon^2} \log\left( \frac{1}{\delta} \right),
\end{align*}
which concludes the proof.
\end{proof}

\subsection{Proof of \Cref{thm:learning}}

\paragraph{Proof roadmap.} The proof proceeds in four main steps. First, \Cref{lemma:error-decomposition} separates the learning error into a reward-discretization term, of order $H\kappa$, and two value-estimation terms for the discretized MDPs $\mathcal M_\kappa$ and $\widehat{\mathcal M}_\kappa$. Second, we reduce these two MDPs to finite-horizon MDPs on the augmented state space $\mathcal X$, whose state records the current original state, the best return obtained in previous rollouts, and the cumulative reward in the current rollout. This reduction allows us to compare $\mathcal M_\kappa$ and $\widehat{\mathcal M}_\kappa$ only through their transition kernels. Third, \Cref{lemma:simulation} and \Cref{lemma:simulation-lemma-part-2} give simulation bounds for the value-estimation errors: the first controls the discretized optimal policy $\mu^\star_\kappa$, while the second controls the data-dependent empirical optimizer $\hat\mu$ through the reference value function of $\mu^\star_\kappa$. Combining these two bounds yields the refined decomposition in \Cref{lemma:refined-error-decomposition}, which reduces the problem to bounding the quantities $C_{\mu^\star_\kappa,\mu^\star_\kappa}$ and $C_{\hat\mu,\mu^\star_\kappa}$. These quantities measure the accumulated one-step transition-estimation error along the trajectory distribution induced by the first policy in the subscript, where the empirical transition error is tested against the continuation value function of the second policy. So $C_{\mu^\star_\kappa,\mu^\star_\kappa}$ controls the estimation error for the fixed discretized optimizer, whereas $C_{\hat\mu,\mu^\star_\kappa}$ controls the data-dependent empirical optimizer through the reference value function $V^{\mu^\star_\kappa,\kappa}$. Fourth, \Cref{lemma:variance-lemma} proves that the cumulative variance of the reference value function along any policy is at most $3H^2$, and \Cref{lemma:high-probability-error-bound} combines this variance estimate with concentration for the empirical transition kernel to obtain a high-probability bound of order $\sqrt{KH^3/m}+KH^2/m+H\kappa$. The theorem then follows by choosing $\kappa=\epsilon/(9H)$ and $m$ large enough so that the statistical and discretization errors are each absorbed into the target accuracy $\epsilon$.

\subsubsection{Error Decomposition and Simulation Lemmas}

We start by introducing some additional notation used throughout the section. For any fixed $\kappa > 0$ and $m \in \Naturals$, we denote by $\hat \mu$ the optimal policy for $\widehat{\mathcal{M}}_\kappa = (\mathcal{S}, \mathcal{A}, s_0, H, \{\hat p _h \}_{h \in [H]}, \{ \tilde r_{h, \kappa} \}_{h \in [H]} )$. We recall that $\hat \mu$ is the policy returned by the algorithm we presented in \Cref{sec:algo}. Furthermore, we denote by $\pi^\star$ the optimal policy for $\mathcal{M}$, and by $\mu^\star_\kappa$ the optimal policy for $\mathcal{M}_\kappa$. 
For any policy $\pi$, let $V^\pi$ and $\widehat V^\pi$ denote the value functions of $\pi$ in $\mathcal{M}$ and $\widehat{\mathcal{M}}$, respectively. Similarly, let $V^{\pi,\kappa}$ and $\widehat V^{\pi,\kappa}$ denote the value functions of $\pi$ in $\mathcal{M}_\kappa$ and $\widehat{\mathcal{M}}_\kappa$, respectively.

The following error decomposition shows that it suffices to control value-estimation errors for the discretized-reward MDPs.

\begin{lemma}[Error Decomposition]\label{lemma:error-decomposition}
    It holds that:
    \begin{align*}
        V^\star - V^{\hat \mu} \le V^{\mu^\star_\kappa, \kappa} - V^{\hat \mu, \kappa}+  3H\kappa \le  \left( V^{\mu^\star_\kappa, \kappa} - \widehat V^{\mu^\star_\kappa, \kappa} \right) + \left( \widehat V^{\hat \mu, \kappa}- V^{\hat \mu, \kappa} \right) + 3H\kappa.
    \end{align*}
\end{lemma}
\begin{proof}
With algebraic manipulations, we have that:
\begin{align*}
    V^\star - V^{\hat \mu} &  = \left( V^\star - V^{\mu^\star_\kappa, \kappa} \right) + \left( V^{\mu^\star_\kappa, \kappa} - V^{\hat \mu, \kappa} \right) + \left( V^{\hat \mu, \kappa}- V^{\hat \mu} \right) \\
    & \le V^{\mu^\star_\kappa, \kappa} - V^{\hat \mu, \kappa}+  3H\kappa   \tag{\Cref{theo:approximation} + \Cref{lemma:same-pol-diff-rew}} \\
    & = V^{\mu^\star_\kappa, \kappa} - \widehat V^{\hat \mu, \kappa} + \widehat V^{\hat \mu, \kappa}- V^{\hat \mu, \kappa}+  3H\kappa \\
    & \le \left( V^{\mu^\star_\kappa, \kappa} - \widehat V^{\mu^\star_\kappa, \kappa} \right) + \left( \widehat V^{\hat \mu, \kappa}- V^{\hat \mu, \kappa} \right) +  3H\kappa \tag{$\hat \mu$ is optimal for $\widehat{\mathcal{M}}_\kappa$},
\end{align*}
which concludes the proof. 
\end{proof}

\paragraph{Reduction to finite horizon MDPs} Given \Cref{lemma:error-decomposition}, the two quantities of interest are $V^{\mu^\star_\kappa, \kappa} - \widehat V^{\mu^\star_\kappa, \kappa}$ and $\widehat V^{\hat \mu, \kappa}- V^{\hat \mu, \kappa}$. Both terms involve only value functions of policies that belong to $\compPol$ evaluated in MDPs that use the discretized reward function, \ie $\mathcal{M}_\kappa$ and $\widehat{\mathcal{M}}_\kappa$. Observe that these MDPs differ only in the transition kernel. As a consequence, in order to bound these terms, it is convenient to reduce $\mathcal{M}_\kappa$ and $\widehat{\mathcal{M}}_\kappa$ to (almost) standard finite-horizon MDPs. Let $\mathcal{X} \coloneqq \mathcal{S} \times \lceil \frac{H}{\kappa} \rceil \times  \lceil \frac{H}{\kappa} \rceil$ be an extended set of states such that each state $x = (s, M, G)\in \mathcal{X}$ can be interpreted as follows: $s \in \mathcal{S}$ is a state of the original state space, $M$ denotes the best return observed in the previous rollouts, and $G$ denotes the current cumulative reward in the current rollout. The action set $\mathcal{A}$ remains the same. Then, for each triple $(x,a,x')\in\mathcal{X}\times\mathcal{A}\times\mathcal{X}$, we define the new transition kernel as follows:
\begin{align*}
    & T_{h, i} (x' \mid x,a) = p_h(s' \mid s,a) \mathds{1} \{ M' = M \} \mathds{1} \{ G' = G + \tilde r_{h,\kappa}(s,a) \} \quad \forall h < H, \forall i \in [K] \\
    & T_{H, i}(x' \mid x,a) = \mathds{1} \{ s' = s_1\} \mathds{1} \{ M' = \max\{M, G + \tilde r_{H,\kappa}(s,a)\} \} \mathds{1}\{ G' = 0 \}\quad \forall i \in [K].
\end{align*}
As one can see, the state evolves according to (1) the original transition kernel $p$ and (2) the discretized reward function $\tilde r$ so that $M$ and $G$ preserve the meaning of the \maxatk objective. Finally, we define the new reward function as follows:
\begin{align*}
    & R_{h,i}(x,a) = 0 \quad \forall (h,i) \ne (H, K)\\
    & R_{H,K}(x,a) = \max \{M, G + \tilde{r}_{H}(s,a)\}.
\end{align*}
At this point, we turned the original problem into a traditional finite-horizon one where the goal is to find a policy that maximizes the standard RL objective in the MDP 
\[
\mathcal{M}^{\text{ext}}_{\kappa} \coloneqq \left(\mathcal{X}, \mathcal{A}, (s_1,0,0), KH, \{T_{h,i}\}_{h \in [H], i \in [K]}, \{R_{h,i} \}_{h \in [H], i \in [K]} \right),
\]
where $(s_1,0,0)\in\mathcal{X}$ is the initial state in the extended state space.

Given this reduction, in the following, we treat the value functions of policies in $\compPol$ as $|\mathcal{X}|$-dimensional vectors (once the step $h$ and the rollout $i$ are fixed). Next, to simplify the notation, it will be convenient to write $(h^+, i^+)$ to denote the (step,rollout) pair that follows $(h,i)$. Precisely, we have that:
\[
(h^+,i^+) =
\begin{cases}
(h+1,i), & \text{if } h < H,\\
(1,i+1), & \text{if } h = H.
\end{cases}
\]
Moreover, for any two pair $(h_1, i_1)$ and $(h_2, i_2)$, we write $(h_1, i_1) \preceq (h_2, i_2)$ if $(h_1, i_1)$ precedes $(h_2, i_2)$. Furthermore, given any policy $\mu \in \compPol$, step $h$ and rollout $i$, we introduce the policy-induced transition kernel $T^{\mu}_{h,i}$, that is defined as $T^{\mu}_{h,i}(x' \mid x) = \sum_{a \in \mathcal{A}} \mu_{h,i}(a \mid x)\, T_{h,i}(x' \mid x,a)$. In this way, we can compactly express the value function as follows:
\begin{align*}
&V^{\mu,\kappa}_{h,i} = T^\mu_{h,i} V^{\mu,\kappa}_{h^+,i^+}, & \forall (h,i) \neq (H,K), \\ &V^{\mu,\kappa}_{H,K}(x) = V^{\mu,\kappa}_{H,K}=\sum_{a\in\mathcal{A}}\mu_{H,K}(a\mid x)\max\{M,G+\tilde r_{H,\kappa}(s,a)\} & \forall x=(s,M,G).
\end{align*}
Finally, before continuing, we extend this construction to the case in which we use an approximation $\hat p$ of $p$. In this case, we denote the extended MDP with $\widehat{\mathcal{M}}_{\kappa}^{\text{ext}}$. This object differs from ${\mathcal{M}}_{\kappa}^{\text{ext}}$ only in the transition kernel, where $T$ is replaced with its empirical estimate $\widehat{T}$. 

Now, we can proceed with the proof. First we introduce some notation. For any pair of policies $\mu, \nu \in \compPol$, we define, for brevity, $D_\mu, D_{\mu,\nu}$ as follows:
\begin{align*}
    & D_{\mu} \coloneqq \max_{\bar h, \bar i,\bar x } \| \widehat T^{\mu}_{\bar h, \bar i}(\cdot \mid \bar x) - T^{\mu}_{\bar h,\bar i}(\cdot \mid \bar x) \|_1 \max_{\bar h, \bar i } \|V^{\mu, \kappa}_{\bar h, \bar i} - \widehat V^{\mu, \kappa}_{\bar h, \bar i} \|_{\infty} \\
    & D_{\mu,\nu} \coloneqq \max_{\bar h, \bar i,\bar x } \| \widehat T^{\mu}_{\bar h, \bar i}(\cdot \mid \bar x) - T^{\mu}_{\bar h,\bar i}(\cdot \mid \bar x) \|_1 \max_{\bar h, \bar i } \|V^{\mu, \kappa}_{\bar h, \bar i} - V^{\nu, \kappa}_{\bar h, \bar i} \|_{\infty}.
\end{align*}
Moreover, for the sake of brevity, for a policy $\mu$ and a fixed $(h,i,x)$, we denote $\E_{\mu}[\cdot \mid X_{h,i} = x]$ by $\E_\mu^x[\cdot]$. At this point, we can proceed by analyzing the error terms that appeared in \Cref{lemma:error-decomposition}. We first prove two ``simulation lemmas'' that will be used to control the error terms of \Cref{lemma:error-decomposition}.

\begin{lemma}[Simulation Lemma 1]\label{lemma:simulation}
    Consider any policy $\mu \in \compPol$. Suppose that:
    \begin{align}\label{eq:simulation-lemma-cond}
        \max_{h,i,x} \| \widehat T_{h,i}^\mu(\cdot \mid x) - T_{h,i}^\mu(\cdot \mid x) \|_1 \le \frac{1}{2KH}.
    \end{align}
    Then, it holds that:
    \begin{align*}
        \max_{h, i} \| V^{\mu, \kappa}_{h,i} - \widehat V^{\mu, \kappa}_{h,i} \|_{\infty} \le 2 \max_{h,i, x} \sum_{(\bar h, \bar i) \succeq (h,i)} \E_{\mu}^x \left[ \Big| \sum_{x' \in \mathcal{X}} (T^\mu_{\bar h, \bar i}(x' \mid X_{\bar h, \bar i}) - \widehat T^{\mu}_{\bar h, \bar i}(x' \mid X_{\bar h, \bar i}) ) V_{\bar h^+, \bar i^+}^{\mu, \kappa}(x') \Big| \right].
    \end{align*}
\end{lemma}
\begin{proof}
    The lemma easily follows from the definition of value functions at $(H,K)$. We now prove it for $(h,i) \ne (H,K)$. Specifically:
    \begin{align*}
        \Big| V^{\mu, \kappa}_{h,i} - \widehat V^{\mu, \kappa}_{h,i} \Big| & =  \Big| T^{\mu}_{h,i} V^{\mu, \kappa}_{h^+,i^+} - \widehat T^{\mu}_{h,i} \widehat V^{\mu, \kappa}_{h^+,i^+} \Big| \\
        & \le \Big| (T^{\mu}_{h,i} - \widehat T^{\mu}_{h,i}) V^{\mu, \kappa}_{h^+,i^+}\Big| +  T^{\mu}_{h,i} \Big| V^{\mu, \kappa}_{h^+,i^+} - \widehat V^{\mu, \kappa}_{h^+,i^+}\Big| + \Big|( \widehat T^\mu_{h, i} - T^{\mu}_{h,i} ) (V^{\mu, \kappa}_{h^+,i^+} - \widehat V^{\mu, \kappa}_{h^+,i^+} )\Big| \\
        & \le \Big| (T^{\mu}_{h,i} - \widehat T^{\mu}_{h,i}) V^{\mu, \kappa}_{h^+,i^+}\Big| +  T^{\mu}_{h,i} \Big|V^{\mu, \kappa}_{h^+,i^+} - \widehat V^{\mu, \kappa}_{h^+,i^+}\Big| + D_{\mu}
    \end{align*}
    Now, fix a state $x \in \mathcal{X}$ and analyze the upper bound on $\Big|V^{\mu, \kappa}_{h,i}(x) - \widehat V^{\mu, \kappa}_{h,i}(x)\Big|$ that we just obtained. Using the previous equation we have that:
    \begin{align*}
        \Big | V^{\mu, \kappa}_{h,i}(x) - \widehat V^{\mu, \kappa}_{h,i}(x) \Big | 
        & = \Big| \sum_{x' \in \mathcal{X}} (T^\mu_{h,i}(x'\mid x) - \widehat T^{\mu}_{h,i}(x'\mid x) ) V_{h^+, i^+}^{\mu, \kappa}(x') \Big| + \E_{T^\mu_{h,i}(x)} \left[ \Big|V^{\mu, \kappa}_{h^+,i^+}(x') - \widehat V^{\mu, \kappa}_{h^+,i^+}(x')\Big|  \right]  + D_\mu.
    \end{align*}
    Unrolling the the r.h.s. of this equation, we have that:
    \begin{align*}
        \Big|V^{\mu, \kappa}_{h,i}(x) - \widehat V^{\mu, \kappa}_{h,i}(x)\Big| \le \sum_{(\bar h, \bar i) \succeq (h,i)} \E_{\mu}^x \left[ \Big| \sum_{x' \in \mathcal{X}} (T^\mu_{\bar h, \bar i}(x' \mid X_{\bar h, \bar i}) - \widehat T^{\mu}_{\bar h, \bar i}(x' \mid X_{\bar h, \bar i}) ) V_{\bar h^+, \bar i^+}^{\mu, \kappa}(x') \Big|  \right] +  KH D_\mu.
    \end{align*}
    Now, taking maximum on both sides, using the definition of $D_\mu$ together with \Cref{eq:simulation-lemma-cond}, and re-arranging the terms yield the result.
\end{proof}

Next, we continue with the second simulation lemma, which bounds the value estimation error for a policy $\mu$ using the value function of a reference policy $\nu$.

\begin{lemma}[Simulation Lemma 2]\label{lemma:simulation-lemma-part-2}
    Consider any policy $\mu, \nu \in \compPol$. Suppose that:
    \begin{align}\label{eq:simulation-lemma-cond-2}
        \max_{h,i,x} \| \widehat T_{h,i}^\mu(\cdot \mid x) - T_{h,i}^\mu(\cdot \mid x) \|_1 \le \frac{1}{2KH}.
    \end{align}
    It holds that:
    \begin{align*}
        \max_{h, i} \| V^{\mu, \kappa}_{h,i} - \widehat V^{\mu, \kappa}_{h,i} \|_{\infty} \le  2\max_{h,i, x} \sum_{(\bar h, \bar i) \succeq (h,i)} \E_{\mu}^x \left[ \Big| \sum_{x' \in \mathcal{X}}  (T^\mu_{\bar h, \bar i}(x' \mid X_{\bar h, \bar i}) - \widehat T^{\mu}_{\bar h, \bar i}(x' \mid X_{\bar h, \bar i}) ) V_{\bar h^+, \bar i^+}^{\nu, \kappa}(x') \Big| \right] + 2KH D_{\mu,\nu}.
    \end{align*}
\end{lemma}
\begin{proof}
    The lemma easily follows from the definition of value functions at $(H,K)$. Now, we  prove it for $(h,i) \ne (H,K)$. We have that:
    \begin{align*}
        \Big| V^{\mu, \kappa}_{h,i} - \widehat V^{\mu, \kappa}_{h,i} \Big| & =  \Big| T^{\mu}_{h,i} V^{\mu, \kappa}_{h^+,i^+} - \widehat T^{\mu}_{h,i} \widehat V^{\mu, \kappa}_{h^+,i^+} \Big| \\
        & \le \Big| (T^{\mu}_{h,i} - \widehat T^{\mu}_{h,i}) V^{\mu, \kappa}_{h^+,i^+}\Big| +  T^{\mu}_{h,i} \Big| V^{\mu, \kappa}_{h^+,i^+} - \widehat V^{\mu, \kappa}_{h^+,i^+}\Big| + \Big|( \widehat T^\mu_{h, i} - T^{\mu}_{h,i} ) (V^{\mu, \kappa}_{h^+,i^+} - \widehat V^{\mu, \kappa}_{h^+,i^+} )\Big| \\
        & \le \Big| (T^{\mu}_{h,i} - \widehat T^{\mu}_{h,i}) V^{\mu, \kappa}_{h^+,i^+}\Big| +  T^{\mu}_{h,i} \Big|V^{\mu, \kappa}_{h^+,i^+} - \widehat V^{\mu, \kappa}_{h^+,i^+}\Big| + D_{\mu}\\
        & \le \Big| (T^{\mu}_{h,i} - \widehat T^{\mu}_{h,i}) V^{\nu, \kappa}_{h^+,i^+}\Big| + \Big| (T^{\mu}_{h,i} - \widehat T^{\mu}_{h,i}) (V^{\mu,\kappa}_{h^+, i^+} - V^{\nu, \kappa}_{h^+,i^+})\Big| +  T^{\mu}_{h,i} \Big|V^{\mu, \kappa}_{h^+,i^+} - \widehat V^{\mu, \kappa}_{h^+,i^+}\Big| + D_{\mu} \\
        & \le \Big| (T^{\mu}_{h,i} - \widehat T^{\mu}_{h,i}) V^{\nu, \kappa}_{h^+,i^+}\Big| +  T^{\mu}_{h,i} \Big|V^{\mu, \kappa}_{h^+,i^+} - \widehat V^{\mu, \kappa}_{h^+,i^+}\Big| + D_{\mu} + D_{\mu, \nu}.
    \end{align*}
    Now, fix a state $x \in \mathcal{X}$ and analyze the upper bound on $\Big|V^{\mu, \kappa}_{h,i}(x) - \widehat V^{\mu, \kappa}_{h,i}(x)\Big|$ that we just obtained. Let $D \coloneqq D_{\mu} + D_{\mu, \nu}$. Then, we can rewrite the previous equation as:
    \begin{align*}
        \Big | V^{\mu, \kappa}_{h,i}(x) - \widehat V^{\mu, \kappa}_{h,i}(x) \Big |
        & = \Big| \sum_{x' \in \mathcal{X}} (T^\mu_{h,i}(x'\mid x) - \widehat T^{\mu}_{h,i}(x'\mid x) ) V_{h^+, i^+}^{\nu, \kappa}(x') \Big| + \E_{T^\mu_{h,i}(x)} \left[ \Big|V^{\mu, \kappa}_{h^+,i^+}(x') - \widehat V^{\mu, \kappa}_{h^+,i^+}(x')\Big|  \right]  + D.
    \end{align*}
    Unrolling the the r.h.s. of this equation, we have that:
    \begin{align*}
        \Big|V^{\mu, \kappa}_{h,i}(x) - \widehat V^{\mu, \kappa}_{h,i}(x)\Big| \le \sum_{(\bar h, \bar i) \succeq (h,i)} \E_{\mu}^x \left[ \Big| \sum_{x' \in \mathcal{X}} (T^\mu_{\bar h, \bar i}(x' \mid X_{\bar h, \bar i}) - \widehat T^{\mu}_{\bar h, \bar i}(x' \mid X_{\bar h, \bar i}) ) V_{\bar h^+, \bar i^+}^{\nu, \kappa}(x') \Big| \right] +  KH D.
    \end{align*}
    Now, taking maximum on both sides, using the definition of $D_\mu$ together with \Cref{eq:simulation-lemma-cond-2} and re-arranging the terms yield the result.
\end{proof}

We now use \Cref{lemma:simulation} and \Cref{lemma:simulation-lemma-part-2} to obtained a refined version of \Cref{lemma:error-decomposition}.

\begin{lemma}[Refined Error Decomposition]\label{lemma:refined-error-decomposition}
    For any pair of policies $\mu, \nu \in \compPol$, let $C_{\mu, \nu}$ be defined as
    \begin{align*}
        C_{\mu, \nu} \coloneqq \max_{h,i, x} \sum_{(\bar h, \bar i) \succeq (h,i)} \E_{\mu}^x \left[ \Big| \sum_{x' \in \mathcal{X}} (T^\mu_{\bar h, \bar i}(x' \mid X_{\bar h, \bar i}) - \widehat T^{\mu}_{\bar h, \bar i}(x' \mid X_{\bar h, \bar i}) ) V_{\bar h^+, \bar i^+}^{\nu, \kappa}(x') \Big|  \right].
    \end{align*}
    Suppose that $\max_{h,i,x} \| \widehat T_{h,i}^{\mu}(\cdot \mid x) - T_{h,i}^\mu(\cdot \mid x) \|_1 \le (4KH)^{-1}$ holds for all $\mu \in \{\mu^\star_\kappa, \hat \mu \}$. Then, we have that:
    \begin{align*}
        \max_{h,i} \| V^\star_{h,i} - V^{\hat \mu}_{h,i} \|_{\infty} \le 4  C_{\mu^\star_\kappa, \mu^\star_\kappa} + 4C_{\hat \mu, \mu^\star_\kappa} + 3H\kappa.
    \end{align*}
\end{lemma}
\begin{proof}
    Let us start from \Cref{lemma:error-decomposition} and analyze $V^{\mu^\star_\kappa, \kappa} - V^{\hat \mu, \kappa}$. For any $(h,i)$ it holds that:
    \begin{align*}
        \max_{h,i} \| V^{\mu^\star_\kappa, \kappa}_{h,i} - V^{\hat \mu, \kappa}_{h,i} \|_{\infty} & \le \| V^{\mu^\star_\kappa, \kappa}_{h,i} - \widehat V^{\mu^\star_\kappa, \kappa}_{h,i} \|_\infty + \| \widehat V^{\hat \mu, \kappa}_{h,i}- V^{\hat \mu, \kappa}_{h,i} \|_{\infty} \\
        & \le 2 C_{\mu^\star_\kappa, \mu^\star_\kappa} + \| \widehat V^{\hat \mu, \kappa}_{h,i}- V^{\hat \mu, \kappa}_{h,i} \|_{\infty} \tag{\Cref{lemma:simulation} with $\mu = \mu^\star_\kappa$} \\
        & \le 2 C_{\mu^\star_\kappa, \mu^\star_\kappa} + 2C_{\hat \mu, \mu^\star_\kappa} + 2KH D_{\hat \mu, \mu^\star_\kappa} \tag{\Cref{lemma:simulation-lemma-part-2} with $\mu=\hat \mu$ and $\nu = \mu^\star_\kappa$} \\
        & = 2 C_{\mu^\star_\kappa, \mu^\star_\kappa} + 2C_{\hat \mu, \mu^\star_\kappa} + 2KH \max_{\bar h, \bar i,\bar x } \| \widehat T^{\hat \mu}_{\bar h, \bar i}(\cdot \mid \bar x) - T^{\hat \mu}_{\bar h,\bar i}(\cdot \mid \bar x) \|_1 \max_{\bar h, \bar i } \|V^{\hat \mu, \kappa}_{\bar h, \bar i} - V^{\mu^\star_\kappa, \kappa}_{\bar h, \bar i} \|_{\infty} \\
        & \le 2 C_{\mu^\star_\kappa, \mu^\star_\kappa} + 2C_{\hat \mu, \mu^\star_\kappa} + \frac{1}{2}\max_{\bar h, \bar i } \|V^{\hat \mu, \kappa}_{\bar h, \bar i} - V^{\mu^\star_\kappa, \kappa}_{\bar h, \bar i} \|_{\infty} \tag{By assumption on $\widehat T^{\hat \mu}$},
    \end{align*}
    where the equality step follows from the definition of $D_{\hat \mu, \mu^\star_\kappa}$. Rearranging the terms leads to:
    \begin{align*}
         \max_{h,i} \| V^{\mu^\star_\kappa, \kappa}_{h,i} - V^{\hat \mu, \kappa}_{h,i} \|_{\infty} \le 4  C_{\mu^\star_\kappa, \mu^\star_\kappa} + 4C_{\hat \mu, \mu^\star_\kappa}.
    \end{align*}
    Combining these results with \Cref{lemma:error-decomposition} yields the result.
\end{proof}

\subsubsection{High Probability Error Bound}

Next, we prove an important auxiliary step that will be subsequently used to upper bound the terms that appeared in \Cref{lemma:refined-error-decomposition}.

\begin{lemma}[Variance Lemma]\label{lemma:variance-lemma}
    For any policy $\mu \in \compPol$, any $(h,i) \in [H] \times [K]$ and any $x \in \mathcal{X}$, it holds that:
    \begin{align*}
        \sum_{(\bar h, \bar i) \succeq (h,i)} \E_{\mu}^x \left[ \Var_{x' \sim T^{\mu}_{\bar h, \bar i}(\cdot | X_{\bar h, \bar i} )} \left( V^{\mu^\star_{\kappa}, \kappa}_{\bar h^+, \bar i^+}(x') \right)  \right] \le 3H^2.
    \end{align*} 
\end{lemma}
\begin{proof}
    In the rest of this proof, in order to simplify the notation, we denote $V_{h, i}^{\mu^\star_\kappa, \kappa}$ by $V_{h, i}$. We start by decomposing the variance. Observe that, for all $y \in \mathcal{X}$ and any $(h,i) \in [H] \times [K]$ it holds that:
    \begin{align*}
        \Var_{x' \sim T^{\mu}_{h, i}(\cdot | y)} \left( V_{ h^+,  i^+}(x') \right) = \sum_{x' \in \mathcal{X}}  T_{h,i}^\mu(x' \mid y) V_{ h^+,  i^+}(x')^2 - \left( \sum_{x' \in \mathcal{X}} T^\mu_{h,i}(x' \mid y) V_{h^+, i^+}(x')  \right)^2.
    \end{align*}
    Therefore, we can rewrite $\sum_{(\bar h, \bar i) \succeq (h,i)} \E_{\mu}^x \left[ \Var_{x' \sim T^{\mu}_{\bar h, \bar i}(\cdot | X_{\bar h, \bar i} )} \left( V_{\bar h^+, \bar i^+}(x') \right)  \right]$ as:
    \begin{align*}
        \sum_{(\bar h, \bar i) \succeq (h,i)} \E_{\mu}^x \left[ \sum_{x' \in \mathcal{X}} T^\mu_{\bar h, \bar i}(x' \mid X_{\bar h, \bar i}) V_{\bar h^+, \bar i^+}(x')^2 - \left( \sum_{x' \in \mathcal{X}} T^\mu_{\bar h, \bar i}(x' \mid X_{\bar h, \bar i}) V_{\bar h^+, \bar i^+}(x') \right)^2 \right]. \\
    \end{align*}
    Let us add and subtract $V_{\bar h, \bar i}(X_{\bar h,\bar i})^2$ within the inner sums over states. Then, we can decompose this equation into the following two terms:
    \begin{align*}
        & A \coloneqq \sum_{(\bar h, \bar i) \succeq (h,i)} \E_{\mu}^x \left[ \sum_{x' \in \mathcal{X}} T^\mu_{\bar h, \bar i}(x' \mid X_{\bar h, \bar i}) V_{\bar h^+, \bar i^+}(x')^2 - V_{\bar h, \bar i}(X_{\bar h, \bar i})^2  \right] \\
        & B \coloneqq \sum_{(\bar h, \bar i) \succeq (h,i)} \E_{\mu}^x \left[ V_{\bar h, \bar i}(X_{\bar h, \bar i})^2 - \left( \sum_{x' \in \mathcal{X}} T^\mu_{\bar h, \bar i}(x' \mid X_{\bar h, \bar i}) V_{\bar h^+, \bar i^+}(x') \right)^2  \right].
    \end{align*}
    We first upper bound $A$. It holds that:
    \begin{align*}
        A & = \sum_{(\bar h, \bar i) \succeq (h,i)} \E_{\mu}^x \left[ \sum_{x' \in \mathcal{X}} T^\mu_{\bar h, \bar i}(x' \mid X_{\bar h, \bar i}) V_{\bar h^+, \bar i^+}(x')^2 - V_{\bar h, \bar i}(X_{\bar h, \bar i})^2  \right] \\
        & = \sum_{(\bar h, \bar i) \succeq (h,i)} \E_{\mu}^x \left[ V_{\bar h^+, \bar i^+}(X_{\bar h^+, \bar i^+})^2 - V_{\bar h, \bar i}(X_{\bar h, \bar i})^2  \right] \\
        & =  \E_{\mu}^x \left[ V_{H, K}(X_{H, K})^2 - V_{h,i}(X_{h,i})^2  \right] \\
        & \le H^2.
    \end{align*}
    Next, we focus on the upper bound for $B$. It holds that:
    \begin{align*}
        B & = \sum_{(\bar h, \bar i) \succeq (h,i)} \E_{\mu}^x \left[ V_{\bar h, \bar i}(X_{\bar h, \bar i})^2 - \left( \sum_{x' \in \mathcal{X}} T^\mu_{\bar h, \bar i}(x' \mid X_{\bar h, \bar i}) V_{\bar h^+, \bar i^+}(x') \right)^2 \right] \\
        & \le 2H \sum_{(\bar h, \bar i) \succeq (h,i)} \E_{\mu}^x \left[ \left( V_{\bar h, \bar i}(X_{\bar h, \bar i}) - \sum_{x' \in \mathcal{X}}T^\mu_{\bar h, \bar i}(x' \mid X_{\bar h, \bar i}) V_{\bar h^+, \bar i^+}(x') \right)  \right] \\
        & = 2H \sum_{(\bar h, \bar i) \succeq (h,i)} \E_{\mu}^x \left[ \left( V_{\bar h, \bar i}(X_{\bar h, \bar i}) - V_{\bar h^+, \bar i^+}(X_{\bar h^+, \bar i^+}) \right)   \right] \\
        & = 2H \,\E_{\mu}^x \left[ V_{h,i}(X_{h,i}) - V_{H,K}(X_{H,K}) \right] \\
        & \le 2H^2.
    \end{align*}
    Here, in the first inequality, we have used that $\E[a^2 - b^2] = \E[(a+b)(a-b)] \le 2H \E[(a-b)]$ with $a,b \in [0,H]$, and $\E[a] \ge \E[b]$. To see that $\E[a] > \E[b]$, recall that:
    \begin{align*}
        \E_{\mu}^x[V_{\bar h, \bar i}(X_{\bar h, \bar i}) ] & = \E_{\mu}^x \left[\max_{a \in \mathcal{A}} \sum_{x' \in \mathcal{X}} T_{\bar h, \bar i}(x' \mid X_{\bar h, \bar i}, a) V_{\bar h^+, \bar i^+}(x') \right]  \\
        & \ge \E_{\mu}^x \left[  \sum_{x' \in \mathcal{X}}T^\mu_{\bar h, \bar i}(x' \mid X_{\bar h, \bar i}) V_{\bar h^+, \bar i^+}(x')   \right],
    \end{align*}
    where the first step follows from the the fact that $V_{h,i}$ is the value function of the optimal policy in $\mathcal{M}_{\kappa}^{\text{ext}}$, \ie $V_{h,i} = V_{h, i}^{\mu^\star_\kappa, \kappa}$.
    Combining these results, we obtained that:
    \begin{align*}
        \sum_{(\bar h, \bar i) \succeq (h,i)} \E_{\mu}^x \left[ \Var_{x' \sim T^{\mu}_{\bar h, \bar i}(\cdot | X_{\bar h, \bar i} )} \left( V^{\mu^\star_{\kappa}, \kappa}_{\bar h^+, \bar i^+}(x') \right)   \right] = A + B \le 3H^2,
    \end{align*}
    which concludes the proof.
\end{proof}

The following lemma provides a high-probability upper bound on the error of algorithm that we proposed. 

\begin{lemma}[High Probability Error Bound]\label{lemma:high-probability-error-bound}
    Suppose that $m$ satisfies
    \begin{align}\label{eq:lemma-cond-on-m}
        m \ge 2 \left( 2 + {128 K^2H^2\log\left( \frac{2SAH}{\delta} \right)} + {128 \,SK^2H^2} \log\left( {384\,SK^2H^2}  \right) \right).
    \end{align}
    With probability at least $1-\delta$, it holds that:
    \begin{align*}
        \max_{h,i} \| V^\star_{h,i} - V^{\hat \mu}_{h,i} \|_{\infty}  \le 8 \sqrt{\frac{2 \log\left( \frac{ 4 |\mathcal{X}|HKA}{\delta} \right) KH^3}{m}} + 8\frac{\log\left( \frac{ 4 |\mathcal{X}|HKA}{\delta} \right)KH^2}{m}+ 3H\kappa.
    \end{align*}
\end{lemma}
\begin{proof}
    We start by defining a good event $\mathcal{E}$ and prove that it holds with high-probability. Next, we will use $\mathcal{E}$ to prove the result of this lemma. 
    Precisely, we define $\mathcal{E}$ as $\mathcal{E} \coloneqq \mathcal{E}_1 \cap \mathcal{E}_2$, where
    \begin{align*}  
        & \mathcal{E}_1 \coloneqq \bigcap_{(s,a,h)} \left\{ \|\hat p_h(\cdot \mid s,a) - p_h(\cdot \mid s,a) \|_1 \le \sqrt{\frac{ 2\log\left( \frac{2SAH}{\delta} \right)+2S\log(6m) }{m}} \right\} \\
        & \mathcal{E}_2 \coloneqq \bigcap_{(x,a,h,i)} \Bigg\{ \left| \sum_{x' \in \mathcal{X}} \left( T_{h,i}(x' \mid x,a) - \widehat T_{h,i}(x' \mid x,a) \right) V^{\mu^\star_\kappa,\kappa}_{h^+, i^+}(x') \right| \\ &\hspace{3.65cm} \le \sqrt{ \frac{ 2 \log\left( \frac{4 |\mathcal{X}|H K A}{\delta} \right) \Var_{x' \sim T_{h,i}(\cdot \mid x,a)} \left(V^{\mu^\star_\kappa,\kappa}_{h^+, i^+}(x')\right)}{m}} + \frac{ 2H \log\left( \frac{4 |\mathcal{X}|H K A}{\delta} \right) }{3m} \Bigg\}.
\end{align*}

    Then, we first prove that $\Prob(\mathcal{E}) \ge 1-\delta$. To this end, observe that $\Prob(\mathcal{E}) = 1 - \Prob(\mathcal{E}^\complement) \ge 1 - \Prob(\mathcal{E}_1^\complement) - \Prob(\mathcal{E}_2^\complement)$. Hence, we now upper bound $\Prob(\mathcal{E}_1^\complement)$ and $\Prob(\mathcal{E}_2^\complement)$.
    For $\Prob(\mathcal{E}_1^\complement)$ and $\Prob(\mathcal{E}_2^\complement)$, we use \Cref{lemma:high-prob-l1} and \Cref{lemma:bernstein} respectively, and obtain
    \begin{align*}
        \Prob(\mathcal{E}_1^\complement) & \le \sum_{s,a,h} \frac{\delta}{2SAH} \le \frac{\delta}{2}, \quad\quad \Prob(\mathcal{E}_2^\complement) \le \sum_{x,a,h} \frac{\delta}{2 |\mathcal{X}|HKA} \le \frac{\delta}{2}.
    \end{align*}
    Therefore, we proved that $\Prob(\mathcal{E})\ge 1-\delta$. 
    Next, we show an important property that holds on $\mathcal{E}$. Precisely, for any $\mu \in \compPol$, state $x \in \mathcal{X}$ and any $(h,i)$, we have that:
    \begin{align*}
        \|T_{h,i}^\mu(\cdot \mid x) - \widehat T_{h,i}^\mu(\cdot \mid x) \|_1 & = \sum_{x' \in \mathcal{X}} \Big| \sum_{a \in \mathcal{A}} \mu_{h,i}(a \mid x)  \left( T_{h,i}(x' \mid x,a) - \widehat T_{h,i}(x' \mid x,a) \right)   \Big| \\
        & \le \sum_{x' \in \mathcal{X}} \sum_{a \in \mathcal{A}}\mu_{h,i}(a \mid x) \Big|   \left( T_{h,i}(x' \mid x,a) - \widehat T_{h,i}(x' \mid x,a) \right) \Big| \\
        & \le \max_{a \in \mathcal{A}} \| T_{h,i}(\cdot  \mid x,a) - \widehat T_{h,i}(\cdot \mid x,a) \|_1  \\
        & = \max_{s,a,h} \| p_h(\cdot \mid s,a) - \hat p_h(\cdot \mid s,a) \|_1 \tag{Def. of $T$} \\
        & \le \sqrt{\frac{2\log\left( \frac{2SAH}{\delta} \right)+2S\log(6m)}{m}} \tag{Def of $\mathcal{E}_1$} \\
        & \le \frac{1}{4KH}. \tag{\Cref{eq:lemma-cond-on-m} and \Cref{lemma:tech-lemma}}
    \end{align*}
    Therefore, we proved that:
    \begin{align}\label{eq:high-prob-ub-eq1}
        \max_{\mu \in \compPol} \max_{h,i} \|T_{h,i}^\mu(\cdot \mid x) - \widehat T_{h,i}^\mu(\cdot \mid x) \|_1  \le \frac{1}{4KH}.
    \end{align}
    Then, since \Cref{eq:high-prob-ub-eq1} holds, we can apply \Cref{lemma:refined-error-decomposition}, and, with probability at least $1-\delta$ we have that:
    \begin{align}\label{eq:high-prob-ub-eq2}
        \max_{h,i} \| V^\star_{h,i} - V^{\hat \mu}_{h,i} \|_{\infty}  \le 4  C_{\mu^\star_\kappa, \mu^\star_\kappa} + 4C_{\hat \mu, \mu^\star_\kappa} + 3H\kappa.
    \end{align}
    In the following, we upper bound $C_{\mu^\star_\kappa, \mu^\star_\kappa}$ and $C_{\hat \mu, \mu^\star_\kappa}$. Before that, for the sake of readability, we introduce some additional notation to simplify the exposition. First, for any pair of policies $\mu, \nu$, we denote by $C_{\mu, \nu}^{h,i,x}$ the following quantity:
    \begin{align*}
        C_{\mu, \nu}^{h,i,x} := \sum_{(\bar h, \bar i) \succeq (h,i)} \E_{\mu}^x \left[ \Big| \sum_{x' \in \mathcal{X}} (T^{\mu}_{\bar h, \bar i}(x' \mid X_{\bar h, \bar i}) - \widehat T^{\mu}_{\bar h, \bar i}(x' \mid X_{\bar h, \bar i}) ) V_{\bar h^+, \bar i^+}^{\nu,\kappa}(x') \Big| \right],
    \end{align*}
    so that $C_{\mu,\nu}= \max_{h,i,x} C_{\mu,\nu}^{h,i,x}$. Furthermore, let $L \coloneqq \log\left( \frac{ 4 |\mathcal{X}|H K A}{\delta} \right) m^{-1}$. Now, we can start to analyze  $C_{\mu, \mu^\star_\kappa}$ for a generic deterministic policy $\mu \in \compPol$. Fix a triplet $(h,i,x)$. We have that:
    \begin{align*}
        C_{\mu, \mu^\star_\kappa}^{h,i,x} & = \sum_{(\bar h, \bar i) \succeq (h,i)} \E_{\mu}^x \left[ \Big| \sum_{x' \in \mathcal{X}} (T^{\mu}_{\bar h, \bar i}(x' \mid X_{\bar h, \bar i}) - \widehat T^{\mu}_{\bar h, \bar i}(x' \mid X_{\bar h, \bar i}) ) V_{\bar h^+, \bar i^+}^{\mu^\star_\kappa, \kappa}(x') \Big| \right] \\
        & \le  \sum_{(\bar h, \bar i) \succeq (h,i)} \E^x_{\mu} \left[ \sum_{a \in \mathcal{A}} \mu(a \mid X_{\bar h, \bar i})\Big| \sum_{x' \in \mathcal{X}} (T_{\bar h, \bar i}(x' \mid X_{\bar h, \bar i},a) - \widehat T_{\bar h, \bar i}(x' \mid X_{\bar h, \bar i},a) ) V_{\bar h^+, \bar i^+}^{\mu^\star_\kappa, \kappa}(x') \Big|\right] \tag{By $|\cdot|$}\\
        & \le \sum_{(\bar h, \bar i)\succeq (h,i)} \E^x_{\mu} \left[ \sum_{a \in \mathcal{A}} \mu(a \mid X_{\bar h, \bar i}) \sqrt{ 2L \Var_{x'\sim T_{\bar h,\bar i}(\cdot \mid X_{\bar h, \bar i},a)} \left(V^{\mu^\star_\kappa,\kappa}_{\bar h^+, \bar i^+}(x')\right)} + LH \right] \tag{By $\mathcal{E}_2$}\\
        & \le \sqrt{ 2L} \sum_{(\bar h, \bar i)\succeq (h,i)} \E^x_{\mu} \left[  \sqrt{ \sum_{a \in \mathcal{A}} \mu(a \mid X_{\bar h, \bar i})  \Var_{x'\sim T_{\bar h,\bar i}(\cdot \mid X_{\bar h, \bar i},a)} \left(V^{\mu^\star_\kappa,\kappa}_{\bar h^+, \bar i^+}(x')\right)}  \right] + LKH^2 \tag{Concavity of $\sqrt{\cdot}$}\\
        & \le \sqrt{ 2L } \sum_{(\bar h, \bar i)\succeq (h,i)} \E^x_{\mu} \left[  \sqrt{ \Var_{x'\sim T_{\bar h,\bar i}(\cdot \mid X_{\bar h, \bar i},a)} \left(V^{\mu^\star_\kappa,\kappa}_{\bar h^+, \bar i^+}(x')\right)}  \right] + LKH^2 \\
        & \le \sqrt{ 2L} \sqrt{ KH   \sum_{(\bar h, \bar i)\succeq (h,i)} \E^x_{\mu} \left[ \Var_{x'\sim T_{\bar h,\bar i}(\cdot \mid X_{\bar h, \bar i},a)} \left(V^{\mu^\star_\kappa,\kappa}_{\bar h^+, \bar i^+}(x')\right) \right]}   + LKH^2 \tag{Cauchy-Schwarz and Jensen}\\
        & \le \sqrt{ 2L KH^3}  + LKH^2. \tag{\Cref{lemma:variance-lemma}}
    \end{align*}
    Plugging this result within \Cref{eq:high-prob-ub-eq2}, we have that:
    \begin{align*}
        \max_{h,i} \| V^\star_{h,i} - V^{\hat \mu}_{h,i} \|_{\infty}  \le 8 \sqrt{2LKH^3} + 8LKH^2 + 3H\kappa,
    \end{align*}
    which concludes the proof.
\end{proof}

\subsubsection{Putting Everything Together}
Now, we can finally use \Cref{lemma:high-probability-error-bound} to prove \Cref{thm:learning}.

\begin{proof}[Proof of \Cref{thm:learning}]
    The proof follows by combining \Cref{lemma:high-probability-error-bound} with an appropriate choice of $m$ and $\kappa$. We first $\kappa = \tfrac{\epsilon}{9H}$. Then, we observe that 
    \[
    |\mathcal{X}| \le S\left(1+\frac H\kappa\right)^2=S\left(1+\frac{9H^2}\epsilon\right)^2.
    \]
    Let
    \[
    L':=\log\left(\frac{4HKAS(1+9H^2/\epsilon)^2}{\delta}\right).
    \]
    Then, we set $m$ as follows:
    \begin{align*}
        m & = \max \left\{ 2 \left( 2 + {128 K^2H^2\log\left( \frac{2SAH}{\delta} \right)} + {128\,SK^2H^2} \log\left( {384 \,SK^2H^2} \right) \right), \frac{1152L'KH^3}{\epsilon^2} \right\} \\
        & \le \widetilde{\mathcal{O}}\left( SK^2H^2 \log\left(\frac{1}{\delta} \right) + \frac{KH^3\log(\frac{1}{\delta})}{\epsilon^2}  \right).
    \end{align*}
    
    Now we can apply \Cref{lemma:high-probability-error-bound} and obtain that, with probability at least $1-\delta$:
    \begin{align*}
        \max_{h,i} \| V^\star_{h,i} - V^{\hat \mu}_{h,i} \|_{\infty} & \le 8 \sqrt{\frac{2 \log\left( \frac{ 4 |\mathcal{X}|HKA}{\delta} \right) KH^3}{m}} + 8 \frac{\log\left( \frac{ 4 |\mathcal{X}|HKA}{\delta} \right)KH^2}{m}+ 3H\kappa \le \epsilon, 
    \end{align*}
    thus concluding the proof.
\end{proof}

%% file: src/appendix_aux.tex
\section{Auxiliary Lemmas}

\begin{lemma}[Application of Bernstein's Inequality]\label{lemma:bernstein}
    Let $D \in \Naturals_{> 0}$ and $q \in \Delta(q)$. Let $\hat q$ be the maximum likelihood estimator of $q$ computed using $n$ independent samples from $q$. Furthermore, let $L > 0$ and $f: D \to [0, L]$ be any bounded function. Then, fix $\delta \in (0,1)$. With probability at least $1-\delta$ it holds that:
    \begin{align*}
        | (q - \hat q)^\top f   | \le \sqrt{\frac{2 \log\left( \frac{2}{\delta} \right) \Var_q(f)}{n}} + \frac{2L \log\left( \frac{2}{\delta} \right)}{3n}.
    \end{align*}
\end{lemma}
\begin{proof}
    Direct application of Bernstein's inequality \citep{boucheron2003concentration}.
\end{proof}

\begin{lemma}[Concentration of $\|\cdot \|_1$ norms]\label{lemma:high-prob-l1}
    Consider $D \in \Naturals_{>0}$ and $q \in \Delta(D)$. Let $\hat q$ be the maximum likelihood estimator of $q$ after $n$ i.i.d. samples. Then, for $\delta \in (0,1)$, it holds that:
    \begin{align*}
        \Prob\left( \| q - \hat q \|_1 \ge \sqrt{\frac{2 \log\left(\frac{1}{\delta} \right)+2D\log(6n)}{n}}\right) \le \delta.
    \end{align*}
\end{lemma}
\begin{proof}
    Direct application of Pinsker's inequality and Proposition 1 in \cite{jonsson2020planning}. 
\end{proof}

\begin{lemma}[Technical Lemma; Lemma D.9 in \cite{barbara2026optimal}]\label{lemma:tech-lemma}
    Let $c_1, c_2, n, L, \delta \in \mathbb{R}_{> 0}$ and consider 
    \[
    \bar{T} = \inf \left\{ t \in \mathbb{N}: \frac{\log(c_1/\delta) + n \log(c_2t)}{t} \le L^2 \right\}.
    \]
    It holds that:
    \[
        \bar{T} \le 2 \left( 2 + \frac{4\log(c_1/\delta)}{L^2} + \frac{4n}{L^2} \log\left( \frac{2nc_2}{L^2} \right)\right).
    \]
\end{lemma}

\begin{lemma}[Lemma 1 in \cite{garivier2019explore}]\label{lemma:com-1}
    Consider a measurable space $(\Omega, \mathcal{F})$ equipped with two distributions $\Prob_1$ and $\Prob_2$. For any $\mathcal{F}$-measurable function $Z: \Omega \to [0,1]$ we have:
    \begin{align*}
        \textup{KL}(\Prob_1, \Prob_2) \ge \textup{kl}(\E_{1}[Z], \E_2[Z]),
    \end{align*}
    where $\E_1$ and $\E_2$ are the expectations under $\Prob_1$ and $\Prob_2$ and $\textup{kl}(p,q)$ denotes the KL-divergence between Bernoulli distributions with mean $p$ and $q$.
\end{lemma}

\begin{lemma}[Lemma 15 in \cite{domingues2021episodic}]\label{lemma:com-2}
    For any $p,q \in [0,1]$:
    \begin{align*}
        \textup{kl}(p,q) \ge (1-p) \log\left( \frac{1}{1-q} \right) - \log(2),
    \end{align*}
    where $\textup{kl}(p,q)$ denotes the KL-divergence between Bernoulli distributions with mean $p$ and $q$.
\end{lemma}